\documentclass{article}

\usepackage{hyphenat}
\usepackage{times}
\usepackage{graphicx} % more modern
\usepackage{subfigure}

\usepackage{natbib}

\usepackage{algorithm}
\usepackage{algorithmic}

\usepackage{amsfonts}
\usepackage{bbm}
\usepackage{multirow}
\usepackage{amsmath,amssymb,amsthm}
\usepackage{hyperref}
\usepackage[english]{babel}

\hypersetup{breaklinks=true}
\newtheorem{theorem}{Theorem}
\newtheorem{lemma}[theorem]{Lemma}
\newtheorem{corollary}[theorem]{Corollary}
\usepackage{mathtools} % Bonus

\DeclareMathOperator*{\argmax}{\arg\!\max}

\newcommand{\R}{\mathbb R}

\newcommand{\calP}{\mathcal P}
\newcommand{\calM}{\mathcal M}
\newcommand{\calS}{\mathcal S}

\newcommand{\calX}{\mathcal X}
\newcommand{\calD}{\mathcal D}
\renewcommand*{\vec}[1]{\boldsymbol{#1}}

\newcommand{\calT}{\mathcal T}

\newcommand{\dvol}{{\rm dvol}}

\newcommand{\maps}{{\rm Maps} }
\newcommand{\trii}{ {\rm Tr\ II}}

\newcommand{\pdist}{\calP_{D}}
\newcommand{\pt}{\calP_{\calT_m}}
\newcommand{\pg}{\calP_G}

\newcommand{\grf}{{\rm gr}(\vec f)}

\newcommand{\sign}{\rm sign}

\newcommand{\mdo}{\mathbbm{1}}
\newcommand{\mon}{\mdo_{h_{\vec f_n}(\vec x)\neq y}}
\newcommand{\mpp}{\mdo_{h_{\vec\eta}(\vec x)\neq y}}
\newcommand{\bbo}{\mdo}
\newcommand{\mx}{\mu(\vec x)d\vec x}
\newcommand{\smax}{\rm submax}

\usepackage[accepted]{icml2016}

\icmltitlerunning{Class Probability Estimation via Differential Geometric Regularization}

\begin{document}

\twocolumn[
\icmltitle{Class Probability Estimation via Differential Geometric Regularization}

% It is OKAY to include author information, even for blind
% submissions: the style file will automatically remove it for you
% unless you've provided the [accepted] option to the icml2016
% package.
\icmlauthor{Qinxun Bai}{qinxun@cs.bu.edu}
\icmladdress{Department of Computer Science, Boston University,
            Boston, MA 02215 USA}
\icmlauthor{Steven Rosenberg}{sr@math.bu.edu}
\icmladdress{Department of Mathematics and Statistics, Boston University,
            Boston, MA 02215 USA}
\icmlauthor{Zheng Wu}{wuzheng@bu.edu}
\icmladdress{The Mathworks Inc. Natick, MA 01760 USA}
\icmlauthor{Stan Sclaroff}{sclaroff@cs.bu.edu}
\icmladdress{Department of Computer Science, Boston University,
            Boston, MA 02215 USA}

% You may provide any keywords that you
% find helpful for describing your paper; these are used to populate
% the "keywords" metadata in the PDF but will not be shown in the document
\icmlkeywords{boring formatting information, machine learning, ICML}

\vskip 0.3in
]

\begin{abstract}
We study the problem of supervised learning for both binary and multiclass classification from a unified geometric perspective. In particular, we propose a geometric regularization technique to find the submanifold corresponding to a robust estimator of the class probability $P(y|\vec x).$
The regularization term measures the volume of this submanifold, based on the intuition that overfitting produces rapid local oscillations and hence large volume of the estimator.
This technique can be applied to regularize any classification function that satisfies two requirements: firstly, an estimator of the class probability can be obtained; secondly, first and second derivatives of the class probability estimator can be calculated.
In experiments, we apply our regularization technique to standard loss functions for classification, our RBF-based implementation compares favorably to widely used regularization methods for both binary and multiclass classification.
\end{abstract}

\section{Introduction}
\label{sec:intro}

%In this work, we introduce a geometric view of overfitting and propose a regularization approach that exploits the geometry of a robust conditional probability estimator for classification.
In supervised learning for classification, the idea of regularization seeks a balance between a perfect description of the training data and the potential for generalization to unseen data. Most regularization techniques are defined in the form of penalizing some functional norms. For instance, one of the most successful classification methods, the support vector machine (SVM)~\cite{Vapnik1998,ScholkopfSmola2002} and its variants~\cite{Bartlett2006,Steinwart2005}, use a RKHS norm as a regularizer. While functional norm based regularization is widely-used in machine learning, we feel that there is important local geometric information overlooked by this approach.

In many real world classification problems, if the feature space is meaningful, then all samples that are locally within a small enough neighborhood of a training sample should have class probability $P(y|\vec x)$ similar to the training sample. For instance, a small enough perturbation of RGB values at some pixels of a human face image should not change dramatically the likelihood of correct identification of this image during face recognition. However, such ``small local oscillations" of the class probability are not explicitly incorporated by penalizing commonly used functional norms. For instance, as reported by~\citet{Goodfellow2014}, linear models and their combinations can be easily fooled by hardly perceptible perturbations of a correctly predicted image, even though a $L2$ regularizer is adopted.

%\begin{figure}[t]
%%\vskip 0.2in
%\centering
%\includegraphics[width=\columnwidth]{pic/dog.eps}
%\caption{Adversarial example from~\cite{Szegedy2013}. Left image is a correctly predicted sample, center image is the difference between the left and the right image magnified by 10 times (values shifted by 128 and clamped), right image is the adversarial example incorrectly predicted by the network.}
%\label{fig:dog}
%\vskip -0.1in
%\end{figure}

Geometric regularization techniques have also been studied in machine learning.~\citet{Belkin2006} employed geometric regularization in the form of the $L2$ norm of the gradient magnitude supported on a manifold. This approach exploits the geometry of the marginal distribution $P(\vec x)$ for semi-supervised learning, rather than the geometry of the class probability $P(y|\vec x)$. Other related geometric regularization methods are motivated by the success of level set methods in image segmentation~\cite{CaiSowmya2007,VarshneyWillsky2010} and Euler's Elastica in image processing~\cite{Lin2012,Lin2015}. In particular, the Level Learning Set~\cite{CaiSowmya2007} combines a counting function of training samples and a geometric penalty on the surface area of the decision boundary. The Geometric Level Set~\cite{VarshneyWillsky2010} generalizes this idea to standard empirical risk minimization schemes with margin-based loss.
 %and carefully treats the variational problem with a Radial Basis Function (RBF) approximation.
Along this line, the Euler's Elastica Model~\cite{Lin2012,Lin2015} proposes a regularization technique that penalizes both the gradient oscillations and the curvature of the decision boundary.
%While achieving empirically competitive performance with SVMs,
However, all three methods focus on the geometry of the decision boundary supported in the domain of the feature space, and the ``small local oscillation" of the class probability is not explicitly addressed.

%\begin{figure}[t]
%%\vskip 0.2in
%\centering
%\includegraphics[width=\columnwidth]{pic/fig1.eps}
%\vskip -0.1in
%\caption{A binary classification ($L=2$) illustration that rapid local oscillation of the class probability estimator leads to overfitting and therefore unstable prediction. Denoted by $x$-axis the feature space $\mathcal X$, $y$-axis the probabilistic simplex $\Delta^1=[0,1]$, then the functional graph of the estimator $\vec f=P(y=\text{face}|\vec x)$ corresponds to a curve in the coordinate plane $\mathcal X\times[0,1]$, and its volume is the length of this curve. The top image is a correctly predicted sample (the green point in $\mathcal X$), the middle image is the adversarial example (the red point in $\mathcal X$). The pixel value difference (10x magnified as the bottom image) between them corresponds to $\Delta$ in $\mathcal X$.}
%\label{fig:oscillation}
%\vskip -0.1in
%\end{figure}

In this work, we argue that the ``small local oscillation" of the class probability actually lies in the product space of the feature domain and the probabilistic output space, and can be characterized by the geometry of a submanifold in this product space corresponding to the class probability. Let $\vec f: \mathcal X\to\Delta^{L-1}$ be a class probability estimator, where $\mathcal X$ is the feature space and $\Delta^{L-1}$ is the probabilistic simplex for $L$ classes. From a geometric perspective, if
we regard $\{(\vec x,\vec f(\vec x))|\vec x\in\mathcal X\}$, the \emph{functional graph} (in the geometric sense) of $\vec f$, as a submanifold in $\mathcal X\times\Delta^{L-1}$, then ``small local oscillations" can be measured by the local flatness of this submanifold.
%And as shown in Figure~\ref{fig:oscillation}, overfitting is related to rapid local oscillation of the estimator $\vec f$, which can be measured by the curvature or the volume (length of the curve for the example in Figure~\ref{fig:oscillation}) of the corresponding submanifold.

In our approach, the learning process %(both binary and multiclass)
can be viewed as a submanifold fitting problem that is solved by a geometric flow method. In particular, our approach finds a submanifold by iteratively fitting the training samples in a curvature or volume decreasing manner without any {\it a priori} assumptions on the geometry of the submanifold in $\mathcal X\times\Delta^{L-1}$.
We use gradient flow methods to find an optimal direction, i.e. at each step we find the vector field pointing in the optimal direction to move $\vec f$.
As we will see in the next section, this regularization approach naturally handles binary and multiclass classification in a unified way, while previous decision boundary based techniques (and most functional regularization approaches) are originally designed for binary classification, and rely on ``one versus one", ``one versus all" or more efficiently a binary coding strategy~\cite{VarshneyWillsky2010} to generalize to multiclass case.

%We study the classification problem in the probabilistic setting.
%Given a sample space  $\mathcal{X}$, a  label space $\mathcal{Y}$, and  a
%finite training set of labeled samples $\calT_m = \{(\vec x_i, y_i)\}_{i=1}^m$, where each training sample is generated i.i.d.~from distribution $P$ over $\mathcal X\times\mathcal Y$, our goal is to find a $h_{\calT_m}:\mathcal X\to\mathcal Y$ such that for any new sample $\vec x\in\mathcal X$, $h_{\calT_m}$ predicts its label  $\hat{y}=h_{\calT_m}(\vec x)$.
%For general (binary or multiclass) classification where $\mathcal Y=\{1,\ldots,L\}$, the optimal generalization risk (Bayes risk) is achieved by the classifier
%$h^*(\vec x)=\argmax\{\eta^\ell(\vec x),\ell\in\mathcal Y\}$,
%where $\vec\eta=(\eta^1,\ldots,\eta^L)$ with $\eta^\ell:\mathcal X\to[0,1]$ being the $\ell^{\rm th}$ class probability, i.e. $\eta^\ell(\vec x) = P(y=\ell|\vec x)$.
%
%As a result,
%our study of regularization follows a ``hybrid" plug-in/ERM scheme~\cite{Audibert2007}, i.e. a regularized loss minimization setup to find a nonparametric estimator $\vec f:\mathcal X\to\Delta^{L-1}$, where $\Delta^{L-1}$ is the standard $(L-1)-$simplex in $\R^L$, $\vec f=(f^1,\ldots,f^L)$ with $\sum_{\ell=1}^L f^\ell=1$, and $f^\ell:\mathcal X\to[0,1]$ is an estimator of $\eta^\ell$, and then ``plug-in" $\vec f$ to get the plug-in classifier
%$h_{\vec f}(\vec x)=\argmax\{f^\ell(\vec x),\ell\in\mathcal Y\}.$

%We show that under some mild assumptions on initialization, our algorithm is universally Bayes consistent.
In experiments, a radial basis function (RBF) based implementation of our formulation compares favorably to widely used binary and multiclass classification methods on %benchmark
datasets from the UCI repository and real-world datasets including the Flickr Material Database (FMD) and the MNIST Database of handwritten digits.

\vskip 0.1in
In summary, our contributions are:
\begin{itemize}
%\item A unified geometric perspective for binary and multiclass classification, with a geometric regularization term using the volume of the corresponding submanifold as a remedy to overfitting for the estimator of $\vec\eta$.
\item A geometric perspective on overfitting and a regularization approach that exploits the geometry of a robust class probability estimator for classification,
\item A unified gradient flow based algorithm for both binary and multiclass classification that can be applied to standard loss functions, and %with respect to 0-1 and other loss functions.
\item A RBF-based implementation that achieves promising experimental results.
\end{itemize}

\section{Method Overview}
\label{sec:overview}

In our work, we propose a regularization scheme that exploits the geometry of a robust class probability estimator and suggest a gradient flow based approach to solve for it. In the follow, we will describe our approach. Related mathematical notation is summarized in Table~\ref{tab:notation}.

\begin{figure}[t]
\vspace{-0.2cm}
\centering
\includegraphics[width=.9\columnwidth]{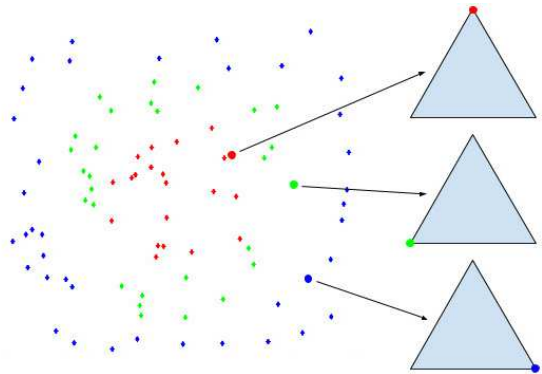}
\vskip -0.1in
\caption{Example of three-class learning, i.e., $L=3$, where the input space  $\mathcal{X}$ is $2d$. Training samples of the three classes are marked with red, green and blue dots respectively.
The class label for each training sample corresponds to a vertex of the simplex $\Delta^{L-1}$.
As a result, each mapped training point $(\vec x_i,\vec z_i)$ lies on one face (corresponding to its label $y_i$) of the space $\cal X\times$$\Delta^2$.
}
\label{fig:3_class}
\vspace{-0.2cm}
\end{figure}

\begin{figure*}[t]
\centering
\begin{tabular}{cccc}
\includegraphics[width=.2\linewidth]{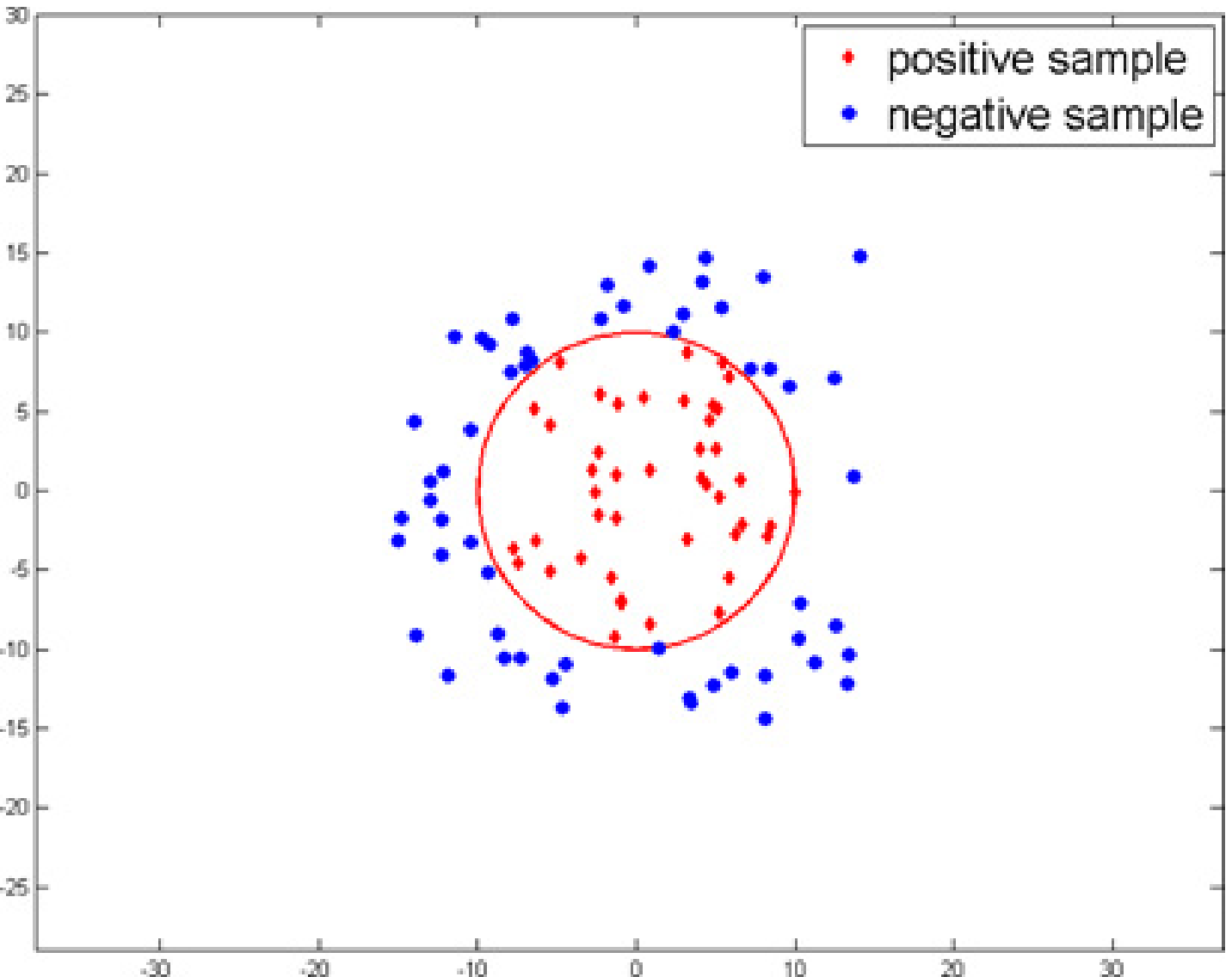}&
\includegraphics[width=.25\linewidth]{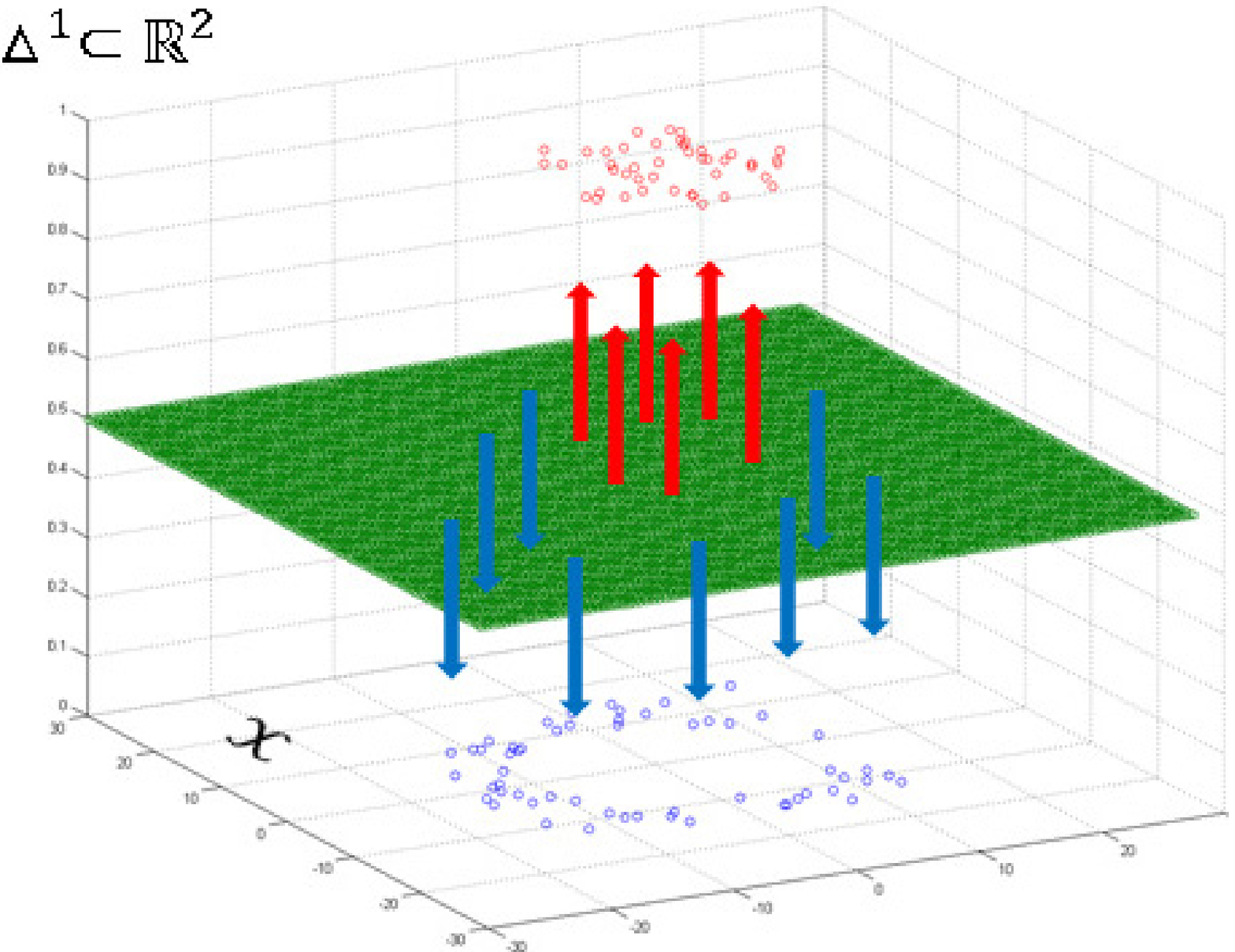}&
\includegraphics[width=.25\linewidth]{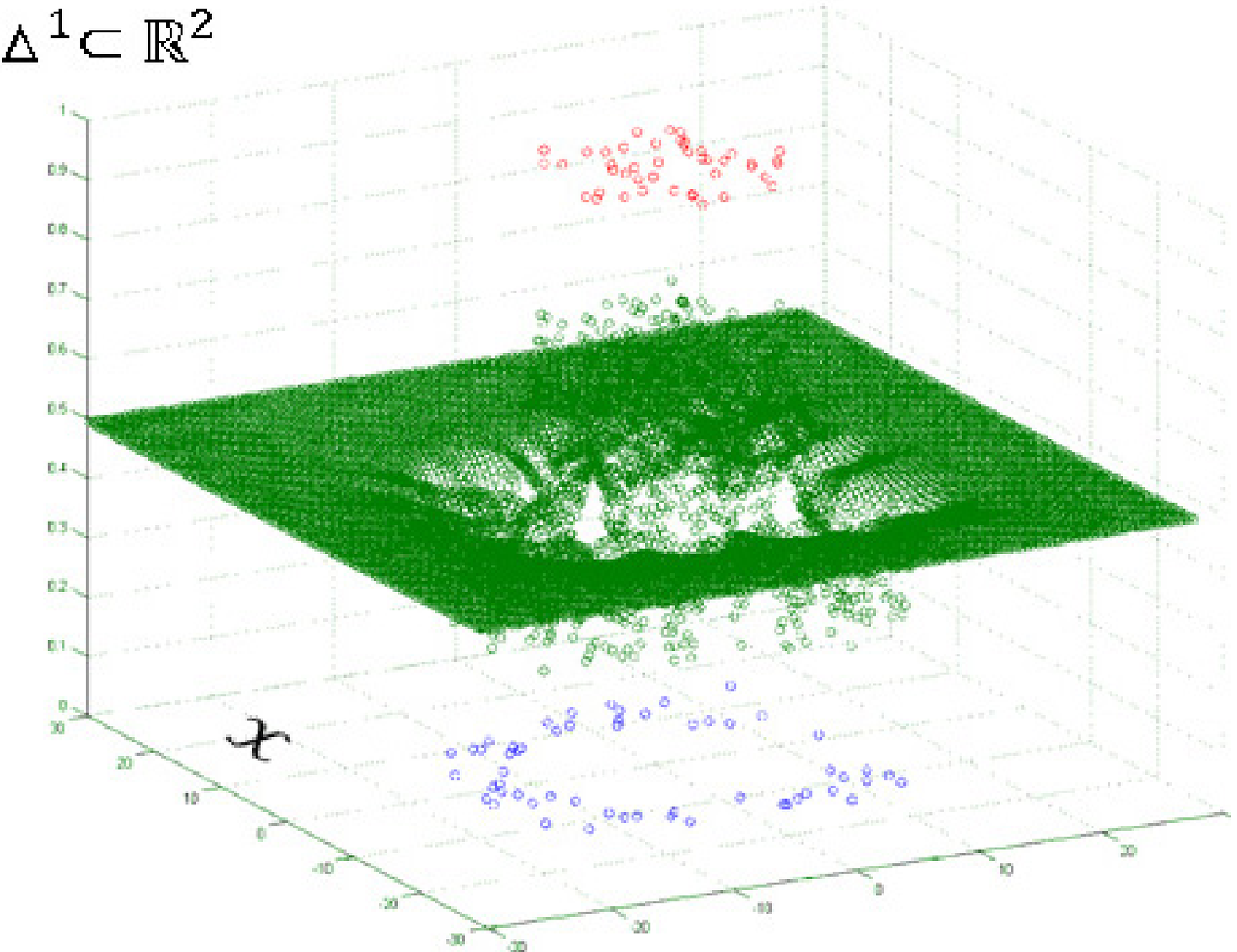}&
\includegraphics[width=.25\linewidth]{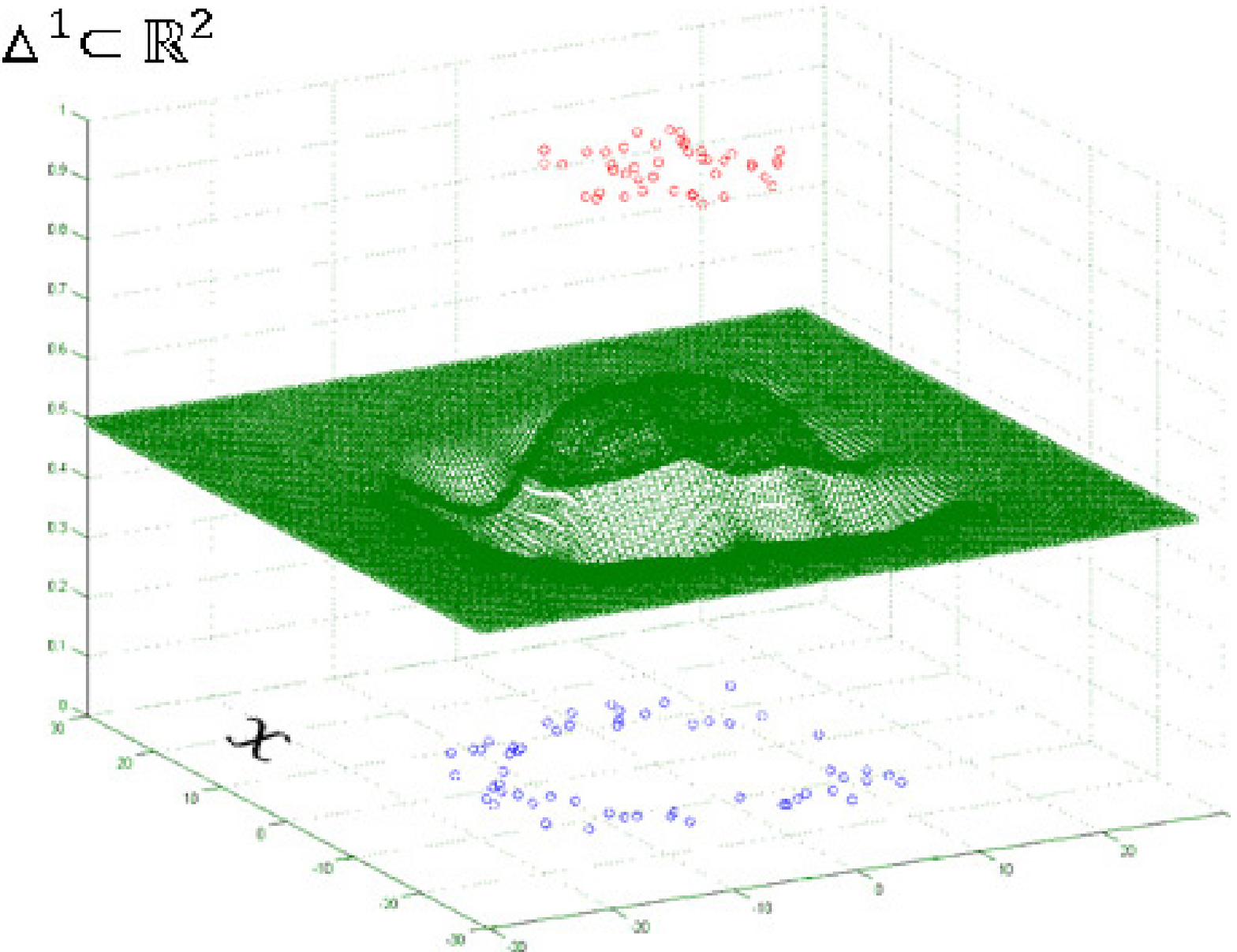}\\
$(a)$ & $(b)$ & $(c)$ & $(d)$\\
\end{tabular}
\caption{Example of binary learning via gradient flow. As shown in $(a),$ the feature space $\cal X$ is $2d$, training points are sampled uniformly within the region $[-15,15]\times [-15,15]$, and labeled by the function $y=\sign(10-\|\vec x\|_2)$ (the red circle). In the initialization step, shown in $(b),$ positive and negative training points map to the two faces of the space $\mathcal X\times\Delta^1$ respectively. Our gradient flow method starts from a neutral function $\vec f_0\equiv\frac{1}{2}$ and moves towards the negative direction (red and blue arrows) of the penalty gradient $\nabla\calP_{\vec f_0}.$ Figure $(c)$ shows the submanifold (${\rm gr}(\vec f_1)$) one step after $(b)$. The submanifold then continues to evolve towards $-\nabla\calP_{\vec f_t}$ step by step and the final output after convergence of the algorithm is shown in $(d).$}
%Note that the volume of $\grf$ is the surface area in this example.}
  \label{fig:2_class}
\vspace{-0.4cm}
\end{figure*}

Following the probabilistic setting of classification,
given a sample (feature) space  $\mathcal{X}\subset\R^N$, a  label space $\mathcal{Y}=\{1,\ldots,L\}$, and  a
finite training set of labeled samples $\calT_m = \{(\vec x_i, y_i)\}_{i=1}^m$, where each training sample is generated i.i.d.~from distribution $P$ over $\mathcal X\times\mathcal Y$, our goal is to find a $h_{\calT_m}:\mathcal X\to\mathcal Y$ such that for any new sample $\vec x\in\mathcal X$, $h_{\calT_m}$ predicts its label  $\hat{y}=h_{\calT_m}(\vec x)$.
The optimal generalization risk (Bayes risk) is achieved by the classifier
$h^*(\vec x)=\argmax\{\eta^\ell(\vec x),\ell\in\mathcal Y\}$,
where $\vec\eta=(\eta^1,\ldots,\eta^L)$ with $\eta^\ell:\mathcal X\to[0,1]$ being the $\ell^{\rm th}$ class probability, i.e. $\eta^\ell(\vec x) = P(y=\ell|\vec x)$.

Our regularization approach exploits the geometry of the class probability estimator, and can be regarded as a ``hybrid" plug-in/ERM scheme~\cite{Audibert2007}. A regularized loss minimization problem is setup to find an estimator $\vec f:\mathcal X\to\Delta^{L-1}$, where $\Delta^{L-1}$ is the standard $(L-1)$-simplex in $\R^L$, and $\vec f=(f^1,\ldots,f^L)$  is an estimator of $\vec\eta$ with $f^\ell:\mathcal X\to[0,1]$. The estimator $\vec f$ is then ``plugged-in" to get the classifier
$h_{\vec f}(\vec x)=\argmax\{f^\ell(\vec x),\ell\in\mathcal Y\}.$

Figure~\ref{fig:3_class} shows an example of the setup of our approach, for a synthetic three-class classification problem. The submanifold corresponding to estimator $\vec f$
is the graph (in the geometric sense) of $\vec f$:  $\grf = \{(\vec x,f^1(\vec x),\ldots,f^L(\vec x)):\vec x\in\mathcal X\}\subset\mathcal X\times\Delta^{L-1}$.
%Here $\Delta^{L-1} = \{(y^1,\ldots, y^L)\in \R^L: y^\ell\geq 0, \sum_\ell y^\ell = 1\}.$
%For example, for $\vec\eta(\vec x)$, we have $\mathcal X\subseteq \R^N, \mathcal Y = \{1,\ldots,L\}$,
%$\mathcal X\times\Delta^{L-1}\subset \R^{N+L}$ with coordinates
We denote a point in the space $\mathcal X\times\Delta^{L-1}$ as $(\vec x,\vec z)=(x^1,\ldots, x^N,z^1,\ldots,z^L)$, where $\vec x\in\mathcal X$ and $\vec z\in\Delta^{L-1}.$
Then in this product space, a training pair $(\vec x_i,y_i=\ell)$ naturally maps to the point
$(\vec x_i,\vec z_i)=(\vec x_i,0,\ldots,1,\ldots,0)$,
with the one-hot vector $\vec z_i$ (with the $1$ in its $\ell$-th slot) at the vertex of
$\Delta^{L-1}$  corresponding to $P(y=y_i|\vec x)=1$.

We point out two properties of this geometric setup. Firstly, it inherently handles multiclass classification, with binary classification as a special case. Secondly, while the dimension of the ambient space, i.e. $\R^{N+L}$, depends on both the feature dimension $N$ and number of classes $L$, the intrinsic dimension of the submanifold $\grf$ only depends on $N$.

%In this simple example, the domain $\mathcal{X}$ is $2d$.  The function value at each point $\vec f(\vec x)$ evolves inside the simplex $\Delta^2$ during training.  In this example, $\vec f(\vec x)$ is initialized to $(\frac{1}{3},\frac{1}{3},\frac{1}{3})$ everywhere.
%With each iteration during training,
%%$\vec f(\vec x)$ is updated in the flow direction that minimizes a data term and curvature term.
%a gradient flow  method is used to move $\vec f(\vec x)$ towards a minimizer of a penalty function that penalizes both the deviation of the submanifold from the training data and large volume.

\subsection{Variational formulation}

We want $\grf$ to approach the mapped training points while remaining as flat as possible,
so we impose a penalty  on $\vec f$ consisting of an empirical loss term $\pt$ and a geometric regularization term $\pg$.
For $\pt$,
%we define a distance penalty on $\grf$ that is an  $L^2$ measure of the distance in $\R^L$ from the $\vec z$ component of a graph point to the averaged $\vec z$ component of its nearest training points.
we can choose either the widely-used cross-entropy loss function for multiclass classification or the simpler Euclidean distance function between the simplex coordinates of the graph point and the mapped training point.
For $\pg$, we would ideally consider an $L^2$ measure of the Riemann curvature of $\grf$, as the vanishing of this term gives optimal (i.e., locally distortion free) diffeomorphisms from $\grf$ to $\R^N.$
However, the Riemann curvature tensor takes the form of a combination of derivatives up to third order, and the corresponding gradient vector field is even more complicated and inefficient to compute in practice.
As a result, we measure the graph's volume,
 $\pg(\vec f) = \int_{\grf} \dvol$, where $\dvol$ is the induced volume
 from the Lebesgue measure on the ambient space $\R^{N+L}$.
%as volume minimizing maps intuitively are not rapidly oscillating.

More precisely, we find the function that minimizes the following penalty $\calP$:
\begin{equation}
\label{eq:variation}
\calP= \pt+\lambda\pg:\calM =\maps(\calX, \Delta^{L-1})\to \R
\end{equation}
on the set $\calM$ of smooth functions from $\calX$ to $\Delta^{L-1}$, where $\lambda$ is the tradeoff parameter between empirical loss and regularization.
It is important to note that any relative scaling of the domain $\mathcal X$
will not affect the estimate of  the class probability $\vec\eta$, as scaling will  distort $\grf$ but will
not change the critical function  estimating $\vec\eta$.

\subsection{Gradient flow and geometric foundation}

The standard technique for solving variational formulas is the Euler-Lagrange PDE. However, due to our geometric term $\pg$, finding the minimal solutions of the Euler-Lagrange equations for $\calP$ is difficult, instead, we solve for ${\rm argmin}\ \calP$ using gradient flow in functional space $\calM.$

A simple but intuitive simulated example of binary learning using gradient flow for our approach is given in Figure~\ref{fig:2_class}.
For the explanation purposes only, we replace $\calM$ with a finite dimensional Riemannian manifold $M$.
% for the moment.
Without loss of generality, we also assume that $\calP$ is smooth, then it has a differential $d\calP_{\vec f}
: T_{\vec f}M\to\R$ for each $\vec f\in M$, where $T_{\vec f}M$ is the tangent space to $M$ at $\vec f$.  Since $d\calP_{\vec f}$ is a linear functional on $T_{\vec f}M$, there is a unique tangent vector, denoted $\nabla \calP_{\vec f}$, such that
$d\calP_{\vec f} (\vec v)= \langle \vec v,\nabla \calP_{\vec f}\rangle$ for all $\vec v\in T_{\vec f}M.$  $\nabla \calP_{\vec f}$ points in the direction of maximal increase of $\calP$ at $\vec f$.  Thus, the solution of the negative gradient flow $d\vec f_t/dt = -\nabla
\calP_{\vec f_t}$ is a flow line of steepest descent starting at an initial $\vec f_0.$  For a dense open set of initial points, flow lines approach a local minimum of $\calP$ at $t\to\infty.$
We always choose the initial function $\vec f_0$ to be the ``neutral" choice $\vec f_0(\vec x) \equiv (\frac{1}{L},\ldots,\frac{1}{L})$ which reasonably assigns equal conditional probability to all classes.  %This choice works well in practice.

Similar gradient flow procedures are widely used in variational problems, such as level set methods~\cite{OsherSethian1988,Sethian1999}, Mumford-Shah functional~\cite{Mumford1989}, etc. In the classification literature, \citet{VarshneyWillsky2010} were the first to use gradient flow methods for solving level set based energy functions, then followed by~\citet{Lin2012,Lin2015} to solve Euler's Elastica models.
In our case,  we are exploiting the geometry in the space $\mathcal X\times\Delta^{L-1},$ rather than standard vector spaces.

Since our gradient flow method is actually applied on the infinite dimensional manifold $\calM,$ we have to understand both the topology and the Riemannian geometry of $\calM$.  For the topology,
we put the Fr\'echet topology on $\calM' = \maps(\calX,\R^L)$, the set of smooth maps from $\mathcal X$ to $\R^L$, and take the induced topology on $\calM.$  Intuitively speaking, two functions in $\calM$ are close if the functions and all their partial derivatives are pointwise close.
Since $\calM$ is an open Fr\'echet submanifold with boundary inside the vector space $\calM'$, so as with an open set in Euclidean space, we can canonically identify $T_{\vec f}\calM$ with $\calM'$.
For the Riemannian metric, we take the $L^2$ metric on each tangent space $T_{\vec f}\calM$: $\langle \phi_1, \phi_2\rangle := \int_{\calX} \phi_1(\vec x)\phi_2(\vec x) \dvol_{\vec x}$, with $\phi_i\in\calM'$ and $\dvol_{\vec x}$ being the volume form of the induced Riemannian metric on the graph of $\vec f$.
(Strictly speaking, the volume form is pulled back to $\calX$ by $\vec f$, usually denoted by $\vec f^*\dvol$.)

The differential $d\calP_{\vec f}$ is
linear as above, and by a direct calculation,  there is a unique tangent vector $\nabla\calP_{\vec f}\in T_{\vec f}\calM$
such that $d\calP_{\vec f}(\phi) = \langle \nabla\calP_{\vec f}, \phi\rangle$ for all $\phi\in T_{\vec f}\calM.$
Thus, we can construct the gradient flow equation.  However, unlike the case of finite dimensions,
the existence of flow lines is not automatic.  Assuming the existence of flow lines,
 a generic initial point flows to a local minimum of $\calP$.  In any case, our RBF-based implementation in \S\ref{sec:method} mimicking gradient flow is well defined.

%In our setup, we want to find the (or a) best estimator $\vec f:\mathcal X\to \Delta^{L-1}$ of $\vec\eta$ on a compact set $\mathcal X\subset \R^N$ given a set of training data $\calT_m = \{(\vec x_i, y_i)\}_{i=1}^m$.
Note that we think of $\mathcal X$ as
large enough so that the training data actually is sampled well inside $\mathcal X$.  This allows us to treat $\mathcal X$ as a closed manifold in our gradient calculations, so that boundary effects can be ignored. A similar \emph{natural boundary condition} is also adopted by previous work~\cite{VarshneyWillsky2010,Lin2012,Lin2015}.

\subsection{More on related work}
\label{sec:related}

There exist some other works that are related to some aspects of our work. Most notably, Sobolev regularization, involves functional norms of a certain number of derivatives of the prediction function. For instance, the manifold regularization~\cite{Belkin2006} mentioned in~\S\ref{sec:intro} uses a Sobolev regularization term,
\begin{equation}
\label{eq:mfld_reg}
\int_{x\in\calM}\|\nabla_{\calM} f\|^2dP(x),
\end{equation}
\vskip -0.1in
where $f$ is a smooth function on manifold $\calM.$ A discrete version of~\eqref{eq:mfld_reg} corresponds to the graph Laplacian regularization~\cite{Zhou2005}.~\citet{Lin2015} discussed in detail the difference between a Sobolev norm and a curvature-based norm for the purpose of exploiting the geometry of the decision boundary.

For our purpose, while imposing, say, a high Sobolev norm\footnote{``High Sobolev norm" is the conventional term for Sobolev norm with high order of derivatives.}, will also lead to a flattening of the hypersurface  $\grf$, these norms are not
specifically tailored to measuring the flatness of   $\grf$. In other words, a high Sobolev norm bound will imply the volume bound we desire, but not {\em vice versa}. As a result, imposing high Sobolev norm constraints (regardless of computational difficulties) overshrinks the hypothesis space from a learning theory point of view. In contrast, our regularization term (given in~\eqref{eq:pgv}) involves only the combination of first derivatives of $\vec f$  that specifically address the geometry behind the ``small local oscillation" prior observed in practice.

Our training procedure for finding the optimal graph of a function is, in a general sense, also related to the manifold learning problem~\cite{Tenenbaum2000,RoweisSaul2000,BelkinNiyogi2003,DonohoGrimes2003,ZhangZha2005,LinZha2008}. The most closely related work is~\cite{DonohoGrimes2003}, which seeks a flat submanifold of Euclidean space that contains a dataset. Again, there are key differences.
Since the goal of~\cite{DonohoGrimes2003} is dimensionality reduction, their manifold has high codimension, while our functional graph has codimension $L-1$, which may be as low as $1$.  More importantly,
we do not assume that the graph of our target function is
a flat (or volume minimizing) submanifold, and we instead flow towards a function whose graph is as flat (or volume minimizing) as possible.
In this regard, our work is related to a large body of literature on Morse theory in finite
%dimensions~\cite{milnor}
and infinite dimensions,
and on mean curvature
flow~\cite{ChenGigaGoto, Mantegazza}.

\section{Example Formulation: RBFs}
\label{sec:method}

We now illustrate our approach using an RBF representation of our estimator $\vec f.$
%We choose to represent $\vec f$ by  radial basis functions (RBFs),
RBFs are also used by previous geometric classification methods~\cite{VarshneyWillsky2010,Lin2012,Lin2015}.

Given values of $\vec f$ are probabilistic vectors, it is common to represent $\vec f$ as a ``softmax" output of RBFs, i.e.
\begin{eqnarray}
\label{eq:softmax}
f^j = \frac{e^{h^j}}{\sum_{l=1}^L e^{h^l}},&&\text{where }
h^j=\sum_{i=1}^m a_i^j\varphi_i(\vec x), \nonumber\\
&&\text{for } j=1,\ldots,L,
\end{eqnarray}
where $\varphi_i(\vec x)=e^{-\frac{1}{c}\|\vec x-\vec x_i\|^2}$ is the RBF function centered at training sample $\vec x_i$, with kernel width parameter $c$.

Estimating $\vec f$ becomes an optimization problem for the $m\times L$ coefficient matrix $A=(a_i^{\ell})$.  The following equation determines $A$:
\begin{equation}
\label{eq:A}
\left[\vec h(\vec x_1),\ldots,\vec h(\vec x_m)\right]^T=GA,\
\text{where } G_{ij}=\varphi_j(\vec x_i).
\end{equation}
To plug this RBF representation into our gradient flow scheme, the gradient vector field $\nabla\calP_{\vec f}$ is evaluated at each sample point $\vec x_i$, and $A$ is updated by
\begin{equation}
\label{eq:upd_A}
A\leftarrow A-\tau G^{-1}\left[\nabla\calP_{\vec h}(\vec x_1),\ldots,\nabla\calP_{\vec h}(\vec x_m)\right]^T,
\end{equation}
where $\tau$ is the step-size parameter, and
\begin{equation}
\label{eq:dfdh}
\nabla\calP_{\vec h}(\vec x_i)
= \left[\frac{\partial\vec f}{\partial\vec h}\right]_{\vec x_i}^T\nabla\calP_{\vec f}(\vec x_i).
\end{equation}
Here $\nabla\calP_{\vec h}(\vec x_i)$ denotes the gradient vector field w.r.t. $\vec h$ evaluated at $\vec x_i,$
and
the $L\times L$ Jacobian matrix $\left[\frac{\partial\vec f}{\partial\vec h}\right]_{\vec x_i}$ can be obtained in closed form from~\eqref{eq:softmax}. In the following subsections, we give exact forms of the empirical penalty $\pt$ and the geometric penalty $\pg$, and discuss the computation of $\nabla\calP_{\vec h}$ for both penalty terms.

\subsection{The empirical penalty $\pt$}
\label{sec:distterm}

We consider two widely-used loss functions for the empirical penalty term $\pt.$

\noindent {\bf Quadratic loss.}
Since $\pt$ measures the deviation of $\grf$ from the mapped training points, it is natural to choose the quadratic function of the Euclidean distance in the simplex $\Delta^{L-1},$
\begin{equation}
\label{eq:pt1}
\pt(\vec f)= \sum_{i=1}^m  \|\vec f(\vec x_i) - \vec z_i \|^2,
\end{equation}
where $\vec z_i$ is the one-hot vector corresponding to the ground truth label of $\vec x_i$.
The gradient vector w.r.t. $\vec f$ evaluated at $\vec x_i$ is
\begin{equation*}
\nabla\calP_{\calT_m,\vec f}(\vec x_i)
= 2(\vec f(\vec x_i) - \vec z_i).
\end{equation*}
The gradient vector w.r.t. $\vec h$ evaluated at $\vec x_i$ is
\begin{equation}
\label{eq:vt1}
\nabla\calP_{\calT_m,\vec h}(\vec x_i) =
2\left[\frac{\partial\vec f}{\partial\vec h}\right]^T_{\vec x_i}(\vec f(\vec x_i) - \vec z_i),
\end{equation}
evaluation of $\left[\frac{\partial\vec f}{\partial\vec h}\right]^T_{\vec x_i}$ is the same as in~\eqref{eq:dfdh}.

\noindent {\bf Cross-entropy loss.}
%We also discuss and implement other choices of $\ell$ in \S\ref{sec:exp}.
The cross-entropy loss function is also widely-used for probabilistic output in classification,
\begin{equation}
\label{eq:pt2}
\pt(\vec f) = -\sum_{i=1}^m\sum_{\ell=1}^L z_i^{\ell}\log f^{\ell}(\vec x_i),
\end{equation}
whose  gradient vector field w.r.t. $\vec h$ evaluated at $\vec x_i$ is
\begin{equation}
\label{eq:vt2}
\nabla\calP_{\calT_m,\vec h}(\vec x_i) = \vec f(\vec x_i) - \vec z_i.
\end{equation}

\subsection{The geometric penalty $\pg$}
\label{sec:geoterm}
As discussed in \S\ref{sec:overview}, we wish to penalize graphs for excessive curvature
%The most informative measure of intrinsic curvature is given by the $L^2$ norm of
%the Riemann curvature tensor, $\int_{\mathcal X} |\calR_{\vec f}|^2  \dvol,$ because the vanishing of this term is equivalent to the graph of $\vec f$ having the same local geometry as Euclidean space.
%As a result,
and we use the following function, which measures the volume of the $\grf$:
\begin{equation}
\label{eq:pgv}
\pg(\vec f) = \int_{\grf}\dvol = \int_{\grf}\sqrt{\det(g)}dx^1\ldots dx^N,
\end{equation}
where $g = (g_{ij})$ with $ g_{ij} = \delta_{ij} + f^a_if^a_j,$ is the Riemmanian metric on $\grf$ induced from the standard dot product on $\R^{N+L}.$ We use the summation convention on repeated indices.
Note that this regularization term is clearly very different from the standard Sobolev norm of any order.

It is standard that
$\nabla\pg = -\trii\in \R^{N+L}$ on the space of all embeddings of $\mathcal X$ in $\R^{N+L}$.
If we restrict to the submanifold of graphs of $\vec f\in \calM'$, it is easy to calculate that the gradient of geometric penalty~\eqref{eq:pgv} is
\begin{equation}
\label{eq:vgf}
\nabla\calP_{G,\vec f} = V_{G,\vec f} = -\trii^L,
\end{equation}
where $\trii^L$ denotes the last $L$ components of $\trii.$ Then the geometric gradient w.r.t. $\vec h$ is
\begin{equation}\label{eq:vg}
\nabla\calP_{G,\vec h} = V_{G,\vec h}
= -\left[\frac{\partial\vec f}{\partial\vec h}\right]^T\trii^L.
\end{equation}
Evaluation of $\left[\frac{\partial\vec f}{\partial\vec h}\right]$ and $\trii^L$ at $\vec x_i$ leads to $\nabla\calP_{G,\vec h}(\vec x_i).$

The formulation given above is general in that it encompasses both the binary and the multiclass cases. For both cases, evaluation of $\left[\frac{\partial\vec f}{\partial\vec h}\right]$ at the training points is the same as that in~\eqref{eq:dfdh}, and evaluation of $\trii^L$ at any point $\vec x$ can be performed explicitly by the following theorem.

\vskip 0.15in
\begin{theorem}
\label{thm:multiclass}
For $\vec f:\R^N\to \Delta^{L-1}$, $\trii^L$\ for\ $\grf$ is given by
\begin{eqnarray}
\label{eq:trii}
\trii^L
&=& (g^{-1})^{ij}
\biggl( f^1_{ji} - (g^{-1})^{rs}  f^a_{rs}f^a_i f^1_j,\ldots, \nonumber \\
 && f^L_{ji} - (g^{-1})^{rs}  f^a_{rs}f^a_i f^L_j\biggr),
\end{eqnarray}
where $f_i^a, f_{ij}^a$ denote partial derivatives of $f^a$.
\end{theorem}
The proof is in Appendix~\ref{sec:multi_proof}.
Note that for our RBF representation~\eqref{eq:softmax}, the partial derivatives $f_i^a, f_{ij}^a$ can be easily obtained in closed form.

%Note that the second line in (\ref{eq:trii}) computes $-V_{G,f}$.
%The proof is in the supplemental material.
\noindent{\bf Simplex constraint.}
The class probability estimators $\vec f:X\to\Delta^{L-1}$ always takes values in  $\Delta^{L-1}\subset \R^L$.
 While this constraint is automatically satisfied for the flow of the empirical gradient vector formula~\eqref{eq:vt1} and~\eqref{eq:vt2}, it may fail for  the flow of geometric gradient vector formula~\eqref{eq:vgf}. There are two ways to enforce this constraint for the geometric gradient vector field. First,  since our initial
 function $\vec f_0$ takes values at the center of $\Delta^{L-1}$, we can orthogonally project the geometric gradient vector $V_{G,\vec f}$  to $V_{G,\vec f}'$ in the tangent space $Z=\{(y^1,\ldots,y^L)\in\R^L:\sum_{\ell=1}^L y^\ell=0\}$ of the simplex, and then scale
$\tau V_{G,\vec f}'$   ($\tau$ is the  stepsize) to ensure that the range of the new $\vec f_1$ lies in
 $\Delta^{L-1}$. We then iterate.
 %Details of this orthogonal projection  and experimental results with it are provided in Appendix~\ref{sec:proj_simplex}.
 More simply, we can select $L-1$ of the $L$ components of $\vec f(\vec x)$,
 call the new function $\vec f':\calX\to R^{L-1}$, %which is equivalent to projecting $\grf$ to $\mathcal X\times\R^{K-1}$,
 and compute the $(L-1)$-dimensional gradient vector $V_{G,\vec f'}$ following~\eqref{eq:vgf} and~\eqref{eq:trii}.
 The omitted component of the desired $L$-gradient vector is determined by $-\sum_{\ell=1}^{L-1}V^\ell_{G,\vec f'}$, by the definition of $Z$. Our implementation reported follows this second
 approach, where we choose  the $(L-1)$ components of $\vec f$ by omitting the component corresponding to the class with least number of training samples.

\subsection{Algorithm summary}
\label{sec:alg}

Algorithm~\ref{alg:classification} gives a summary of the classifier learning procedure.
Input to the algorithm is the training set $\calT_m,$ RBF kernel width $c$, trade-off parameter $\lambda$, and step-size parameter $\tau.$
For initialization, our algorithm first initializes the function values of $\vec h$ and $\vec f$ for every training point,
%In setting the kernel width $\sigma_i$ for the RBF functions, our algorithm adopts the $p$-nearest heuristic suggested by~\cite{BenoudjitVerleysen2003} with an extra constant $c$ over all RBFs:
%\begin{equation}
%\label{eq:width}
%\sigma_i=\frac{c}{p}\bigg(\sum_{j=1}^p\|\vec x_i-\vec x_j(i)\|^2\bigg)^{\frac{1}{2}},
%\end{equation}
%where $\vec x_j(i)$ is the $j$-th nearest neighbor to $\vec x_i$ and we fix $p=5$.
and then constructs matrix $G$ and solves for $A$ by~\eqref{eq:A}.
In the subsequent steps, at each iteration, our algorithm first evaluates the gradient vector field $\nabla\calP_{\vec h}$ at every training point, then updates coefficient matrix $A$ by~\eqref{eq:upd_A}.
For the overall penalty function $\calP = \calP_{\calT_m}+\lambda\pg$, we compute the total gradient vector field $\nabla\calP_{\vec h}$ evaluated at $\vec x_i$,
\begin{eqnarray}
\label{eq:vtot}
&&\nabla\calP_{\vec h}(\vec x_i) =  \nabla\calP_{\calT_m,\vec f}(\vec x_i)
    + \lambda \nabla\calP_{G,\vec f}(\vec x_i) \\
&=& \left\{ \begin{array}{l}
    \left[\frac{\partial\vec f}{\partial\vec h}\right]^T_{\vec x_i}
    \bigg(2(\vec f(\vec x_i) - \vec z_i)
    -\lambda\trii^L_{\vec x_i}\bigg), \text{ quadratic} \nonumber \\
    \vec f(\vec x_i) - \vec z_i
    -\lambda\left[\frac{\partial\vec f}{\partial\vec h}\right]^T_{\vec x_i}
    \trii^L_{\vec x_i},\quad \text{ cross-entropy}.\end{array}\right.
 \end{eqnarray}
Our algorithm iterates until it converges or reaches the maximum iteration number.

The same algorithm applies to both the quadratic loss and the cross-entropy loss. To evaluate the total gradient vectors $\nabla\calP_{\vec h}(\vec x_i)$ in each iteration, for the quadratic loss, we use~\eqref{eq:vt1} and~\eqref{eq:vg} to compute the total gradient vector ~\eqref{eq:vtot}; for the cross-entropy loss, we use~\eqref{eq:vt2} and~\eqref{eq:vg} instead. The remaining steps of the procedure are exactly the same for both loss functions.

The final predictor learned by our algorithm is given by
\begin{equation}
\label{eq:classifier}
F(\vec x)=\argmax\{f^{\ell}(\vec x),\ell\in\{1,2,\cdots,L\}\}.
\end{equation}

%\medskip
%For more information on implementation details and hyper-parameter settings, see Appendix~\ref{sec:details}.

\begin{algorithm}[tbh]
   \caption{Geometric regularized classification}
   \label{alg:classification}
\begin{algorithmic}
   \STATE {\bfseries Input:} training data $\calT_m = \{(\vec x_i, y_i)\}_{i=1}^m$, RBF kernel width $c$, trade-off parameter $\lambda$, step-size $\tau$
   \STATE {\bfseries Initialize:} $\vec h(\vec x_i) = (1,\ldots,1),
   \vec f(\vec x_i) = (\frac{1}{L},\ldots,\frac{1}{L}),$ $\forall i\in\{1,\cdots,m\}$, construct matrix $G$ and solve $A$ by~\eqref{eq:A} \FOR{$t=1$ {\bfseries to} $T$}
   \STATE \begin{itemize}
       \item[--] Evaluate the total gradient vector $\nabla\calP_{\vec h}(\vec x_i)$ at every training point according to~\eqref{eq:vtot}.
       \item[--] Update the $A$ by~\eqref{eq:upd_A}.
       \end{itemize}
   \ENDFOR
   \STATE {\bfseries Output:} class probability estimator $\vec f$ given by~\eqref{eq:softmax}.
\end{algorithmic}
\end{algorithm}

\section{Experiments}
\label{sec:exp}

\begin{table*}[htb]
\caption{Tenfold cross-validation error rate (percent) on four binary and four multiclass classification datasets from the UCI machine learning repository. $(L,N)$ denote the number of classes and input feature dimensions respectively.  We compare both the quadratic loss version (Ours-Q) and the cross-entropy loss version (Ours-CE) of our method with 6 RBF-based classification methods and (or) geometric regularization methods: SVM with RBF kernel (SVM), Radial basis function network (RBN), Level learning set classifier~\cite{CaiSowmya2007} (LLS), Geometric level set classifier~\cite{VarshneyWillsky2010} (GLS), Import Vector Machine~\cite{Zhu2005} (IVM), Euler's Elastica classifier~\cite{Lin2012,Lin2015} (EE).
The mean error rate averaged over all eight datasets is shown in the bottom row. Top performance for each dataset is shown in bold.}
\begin{center}
\begin{small}
%\vskip -0.1in
\begin{sc}
%\scalebox{0.94}{
\begin{tabular}{|l|c|c|c||c|c|c||c|c|}
        \hline
        Dataset$(L,N)$ &  RBN & SVM & IVM & LLS & GLS & EE & Ours-Q & Ours-CE \\
        \hline \hline
        Pima(2,8)     & 24.60 & 24.12 & 24.11 & 29.94 & 25.94 & {\bf 23.33} & 23.98 & 24.51 \\ \hline
        WDBC(2,30)    & 5.79 & 2.81 & 3.16 & 6.50  & 4.40 & {\bf 2.63} & {\bf 2.63} & {\bf 2.63} \\ \hline
        Liver(2,6)    & 35.65 & 28.66 & 29.25 & 37.39 & 37.61 & 26.33 & {\bf 25.74} & 26.31  \\ \hline
        Ionos.(2,34)  & 7.38 & {\bf 3.99}  & 21.73 & 13.11 & 13.67 & 6.55 & 6.83 &  6.26  \\ \hline
  %      avg           & 18.36 & 14.90 & 19.56 & 21.74 & 20.41 & 14.71 & 14.80 & 14.93 \\ \hline
        \hline
        Wine(3,13)    & 1.70 & 1.11 & 1.67  & 5.03  & 3.92 & 0.56 & {\bf 0.00} & {\bf 0.00}  \\ \hline
        Iris(3,4)     & 4.67 & {\bf 2.67} & 4.00 &  3.33 &  6.00 & 4.00 & 3.33 & 3.33  \\ \hline
        Glass(6,9)    & 34.50 & 31.77 & {\bf 29.44} &  38.77 & 36.95 & 32.28 & 29.87 & {\bf 29.44} \\ \hline
        Segm.(7,19)   & 13.07  & 3.81 & 3.64 & 14.40 & 4.03 & 8.80 & {\bf 2.47} & 2.73  \\ \hline
  %      avg           & 13.49  & 9.84 & 9.69 & 15.38 & 12.73 & 11.41 & 8.92 & 8.88 \\ \hline
        \hline
        all-avg       & 15.92  & 12.37 & 14.63 & 18.56 & 16.57 & 13.06 & {\bf 11.86} & 11.90 \\ \hline
        \end{tabular}
 %       }
\label{tab:error}
\end{sc}
\end{small}
\end{center}
%\vskip -0.2in
\end{table*}

To evaluate the effectiveness of the proposed regularization approach, we compare our RBF-based implementation
with two groups of related classification methods. The first group of methods are standard RBF-based methods that use different regularizers than ours. The second group of methods are previous geometric regularization methods.

In particular, the first group includes the Radial Basis Function Network (RBN), SVM with RBF kernel (SVM) and the Import Vector Machine (IVM)~\cite{Zhu2005} (a greedy search variant of the standard RBF kernel logistic regression classifier). Note that both SVM and IVM use RKHS regularizers and the IVM also uses the similar cross-entropy loss as Ours-CE.

The second group includes the Level Learning Set classifier~\cite{CaiSowmya2007} (LLS), the Geometric Level Set classifier~\cite{VarshneyWillsky2010} (GLS) and the Euler's Elastica classifier~\cite{Lin2012,Lin2015} (EE). Note that both GLS and EE use RBF representations and EE also uses the same quadratic distance loss as Ours-Q.

We test both the quadratic loss version (Ours-Q) and the cross-entropy loss version (Ours-CE) of our implementation.

%with widely used regularization methods for classification. In particular, we focus on state-of-the-art RBF-based classification methods with a variety of regularizers, including SVM with RBF kernel (SVM), RBF kernel logistic regression classifier, and two recent geometric regularization methods that also use RBFs, i.e., the geometric level set classifier (GLS)~\cite{VarshneyWillsky2010} and the Euler's Elastica classifier (EE)~\cite{Lin2012}. GLS adopts a regularizer that penalizes the surface area of the decision boundary. EE adopts a regularizer that combines a $1$-Sobolev norm and a curvature penalty on the decision boundary. There is also a degenerate total variation (TV) version of EE with the $1$-Sobolev norm only, which was reported to be inferior to the EE version~\cite{Lin2012}. As a result, we did not include TV. For SVM we use the LIBSVM library. For kernel logistic regression (KLR), we use a revised implentation~\cite{Roscher2012} of the Import Vector Machine (IVM)~\cite{Zhu2005}, which is a greedy search variant of the standard KLR.

\subsection{UCI datasets}
\label{sec:UCI}

We tested our classification method on four binary classification datasets and four multiclass classification datasets.
Given that~\citet{VarshneyWillsky2010} has covered several methods on our comparing list and their implementation is publicly available,
we choose to use the same datasets as~\cite{VarshneyWillsky2010} and carefully follow the exact experimental setup. %by looking into their codes.
Tenfold cross-validation error is reported.
%The average test error is computed using tenfold cross-validation.
%
For each of the ten folds, the kernel-width constant $c$ and tradeoff parameter $\lambda$ are found using fivefold cross-validation on the training folds.
All dimensions of input sample points are normalized to a fixed range $[0,1]$ throughout the experiments.
%using the nine-tenths of the full dataset which is the training data for that fold.
We select $c$ from the set of values $\{1/2^5,1/2^4,1/2^3,1/2^2,1/2,1,2,4,8\}$ and $\lambda$ from the set of values $\{1/1.5^4,1/1.5^3,1/1.5^2,1/1.5,1,1.5\}$ that minimizes the fivefold cross-validation error.
%(As noted above, scaling of the input space will not affect the estimate of $P(y=\ell|\vec x)$.)
The step-size $\tau=0.1$ and iteration number $T=5$ are fixed over all datasets. We used the same settings for both loss functions.

Table~\ref{tab:error} reports the results of this experiment. The top performer for each dataset is marked in bold, and the averaged performance of each method over all testing datasets is summarized in the bottom row.
The numbers for RBN, LLS and GLS are copied from Table 1 of~\cite{VarshneyWillsky2010}. Results for SVM and IVM are obtained by running publicly available implementations for SVM~\cite{ChangLin2011} and IVM~\cite{Roscher2012}. Results for EE are obtained by running an implementation provided by the authors of~\cite{Lin2012}. When running these implementations, we followed the same experimental setup as described above and exhaustively searched for the optimal range for the kernel bandwidth and the trade-off parameter via cross-validation.

As shown in the last row of Table~\ref{tab:error}, two versions of our approach are overall the top two performers among all reported methods. In particular, Ours-Q attains top performance on four out of the eight benchmarks, Ours-CE attains top performance on three out of the eight benchmarks.
%both the quadric loss version and the cross-entropy version, achieves
%%with curvature regularization
%top-two performance on most of the benchmarks, and overall
%%our method
%are the strongest performers among all reported methods.
The performance of the two versions of our method are very close, which shows the robustness of our geometric regularization approach cross different loss functions for classification.
Note that three pairs of comparisons, IVM vs Ours-CE, GLS vs Ours-Q/Ours-CE, and EE vs Ours-Q are of particular interest. We are going to discuss them in detail respectively.

The IVM method of kernel logistic regression uses the same RBF-based implementation and very similar cross-entropy loss as our cross-entropy version Ours-CE, and both methods handle the multiclass case inherently. The main difference lies in regularization, i.e., the standard RKHS norm regularizer vs our geometric regularizer. Ours-CE outperforms IVM on six of the eight benchmars in Table~\ref{tab:error}, and achieves equal performance on one of the remaining two, and is only slightly behind on ``PIMA". The overall superior performance of Ours-CE demonstrates the advantage of the proposed geometric regularization over the standard RKHS norm regularization.

The GLS method uses the same RBF-based implementation as ours and also exploits volume geometry for regularization. As described in~\S\ref{sec:intro}, however, there are key differences between the two regularization techniques. GLS measures the volume of the decision boundary supported in $\mathcal X,$ while our approach measures the volume of a submanifold supported in $\mathcal X\times\Delta^{L-1}$ that corresponds to the class probability estimator. Our regularization technique handles the binary and multiclass cases in a unified framework, while the decision boundary based techniques, such as GLS (and EE), were inherently designed for the binary case and rely on a binary coding strategy to train $\log_2 L$ decision boundaries to generalize to the multiclass case. In our experiments, both Ours-Q and Ours-CE outperform GLS on all the benchmarks we have tested. This demonstrates the effectiveness of exploiting the geometry of the class probability in addressing the ``small local oscillation" for classification.

\begin{table*}[!htbp]
\caption{Notations}
\begin{center}
\begin{small}
\vskip 0.1in
%\begin{sc}
%\scalebox{0.94}{
\bgroup
\def\arraystretch{1.2}
\begin{tabular}{|c|}
        \hline
        $h_{\vec f}(\vec x) = \argmax\limits_{\ell\in\mathcal Y}f^\ell(\vec x), : \text{plug-in classifier of}\ \vec f:\mathcal X\to\Delta^{L-1}$    \\
        \hline

        $\Delta^{L-1} : \text{the standard} (L-1)\text{-simplex in}\  \R^L;$  \\
        \hline

        $\vec\eta(\vec x) = (\eta^1(\vec x),\ldots,\eta^L(\vec x)) :
             \text{class probability:}\ \eta^\ell(\vec x) = P(y=\ell|\vec x)$ \\
        \hline

        $\calM : \{\vec f:\calX\to \Delta^{L-1}): \vec f\in C^\infty\}$   \\
        \hline

        $\calM' : \{\vec f:\calX\to \R^L: \vec f\in C^\infty\}$ \\
        \hline

        $T_f\calM : \text{the tangent space to}\ \calM \text{ at some}\ f\in \calM; T_f\calM\simeq\calM'$ \\
        \hline

        $\text{The graph of}\ f\in \calM\ \text{(or}\ \calM'\text{)} : \grf = \{(\vec x, f(\vec x)):\vec x\in \calX\}$\\
        \hline

        $
        g_{ij} =\frac{\partial f}{\partial x^i}\frac{\partial f}{\partial x^j} : \text{The Riemannian metric on}\ \grf \text{induced from the standard dot product on } \R^{N+L}
        $ \\
        \hline

        $(g^{ij}) = g ^{-1}, \ \text{with}\ g = (g_{ij})_{i,j = 1,\ldots,N}$ \\
        \hline

        $
        \dvol = \sqrt{\det(g)}dx^1\ldots dx^N, \text{the volume element on}\ \grf
        $\\
        \hline

        $\{e_i\}_{i=1}^N : \text{a smoothly varying orthonormal basis of the                                          tangent spaces } T_{(\vec x,\vec f(\vec x)}\grf \text{of the graph of }\vec f
         $\\
        \hline

         $
         \begin{array}{rl}
           \trii:&\text{the trace of the second fundamental form of } \grf,
                  \trii\in \R^{N+L} \\
           \trii = \left(\sum_{i=1}^N D_{e_i}e_i\right)^\perp: & \text{with}\ \perp\ \text{the orthogonal projection to the subspace perpendicular to the}\\
                 &\text{tangent space of}\ \grf \text{ and }
                 D_yw \ \text{the directional derivative of}\ w\ \text{in $y$ direction}
        \end{array}$ \\
        \hline

        $
        \trii^L :
        \text{the projection of $\trii$ onto the last $L$ coordinates of $\R^{N+L}$}
        $ \\
        \hline

        $
            \nabla\calP : \text{the gradient vector field of a function}\ \calP:\calM\to\R
                      \text{ on a possibly infinite dimensional manifold}\ \calM
        $ \\
        \hline
\end{tabular}
\egroup
 %       }
\label{tab:notation}
%\end{sc}
\end{small}
\end{center}
\vskip -0.2in
\end{table*}

The EE method of Euler's Elastica model uses the same RBF-based implementation and the same quadratic loss as our quadratic loss version Ours-Q. The main difference, again, lies in regularization, i.e., a combination of $1$-Sobolev norm and curvature penalty on the decision boundary vs our volume penalty on the submanifold corresponding to the class probability estimator.
Since EE adopts a combination of sophisticated geometric measures on the decision boundary, which fit specifically the binary case, it achieves top performance on binary datasets. However, as explained in~\S\ref{sec:intro}, the geometry of the class probability for general classification, which is captured by our approach, cannot be captured by decision boundary based techniques. That is the reason why Ours-Q, a general scheme for both the binary and multiclass case, outperforms EE on all four multiclass datasets, while it still achieves top performance on binary datasets.
%Both methods achieve comparably top performance on binary datasets. This demonstrates the power of geometric techniques in regularization. EE adopts a combination of sophisticated geometric measures on the decision boundary, which fits specifically the binary case. Ours-Q uses a single but also sophisticated regularizer as an example case of our general approach, and achieves equally promising performance on binary classification. On the other hand, for the multiclass case, our method outperforms EE on all four datasets we have tested.
This again demonstrates our geometric perspective and regularization approach that exploits the geometry of the class probability.

\subsection{Real-world datasets}
\label{sec:FMD}

To test the scalability of our method to high dimensional and large-scale problems, we also conduct experiments on two real-world datasets, i.e., the Flickr Material Database~\cite{FMD} for image classification and the MNIST~\cite{MNIST} Database of handwritten digits.

\noindent {\bf FMD (4096 dimensional).}
The FMD dataset contains 10 categories of images with 100 images per category.
%some examples of images from this dataset are provided in the supplemental material.
We extract image features using the SIFT descriptor augmented by its feature coordinates, implemented by the VLFeat library \cite{VLFeat}. With this descriptor, Bag-of-visual-words uses 4096 vector-quantized visual words, histogram square rooting, followed by L2 normalization.
We compare our method with an SVM classifier with RBF kernels,
using exactly the same 4096 dimensional feature.
Our method achieves a correct classification rate of $48.8\%$ while the SVM baseline achieves
%$44.6\%$. We have also tested the one-vs-one scheme and the SVM with RBF kernel as alternative baselines, which achieves a correct classification rate of
$46.4\%$.
%We have tested both one-vs-one and one-vs-all schemes for the SVM baseline and reported the better one.
%Our method is comparable with the state-of-the-art classification methods on this challenging real-world classification problem, which also shows that our geometric regularization approach scales to high dimensional problems.
Note that while recent
works (Qi et al., 2015; Cimpoi et al., 2015) report better
performance on this dataset, the effort focuses on better feature design, not on the classifier itself.
The features used in those works, such as local texture descriptors and CNN features, are more sophisticated.

\noindent {\bf MNIST (60,000 samples).}
The MNIST dataset contains 10 classes ($0\thicksim 9$) of handwritten digits with $60,000$ samples for training and $10,000$ samples for testing. Each sample is a $28\times 28$ grey scale image.
% For preprocessing, we follow~\cite{Lecun1998} to apply PCA on the raw image and use 40 principle components as input features.
We use $1000$ RBFs to represent our function $\vec f$, with RBF centers obtained by applying K-means clustering on the training set.
Note that our learning and regularization approach still handles all the $60,000$ training samples as described by Algorithm~\ref{alg:classification}.
Our method achieves an error rate of $2.74\%$. While there are many results reported on this dataset, we feel that the most comparable method with our representation is the Radial Basis Function Network with $1000$ RBF units~\cite{Lecun1998}, which achieves an error rate of $3.6\%$.
This experiment shows the potential that our geometric regularization approach scales to larger datasets.

%Note that while there exist literature reporting better performance on this dataset, the feature used is more sophisticated, and thus, their results are not directly comparable.
%Note that while recent works~\cite{Qi2015,Cimpoi2015} report better performance on this dataset, the features used in those works are more sophisticated.

\section{Discussion}

%\subsection{Connection with physical models}
%\label{sec:physics}

Our geometric regularization approach can also be viewed as a combination of common physical models.
As illustrated in Figure~\ref{fig:3_class} and~\ref{fig:2_class}, each training pair $(\vec x, y)$ corresponds to a point at one of the vertices of the simplex associated with $\vec x.$ As a result, all training data lie on the boundary of the space $\mathcal X\times\Delta^{L-1}$, while the functional graph of a class probability estimator $\vec f$ is a hypersurface (submanifold) in
$\mathcal X\times\Delta^{L-1}.$ An initial estimator without training information corresponds to the  flat hyperplane in the neutral position. In response to the presence of the training data, this neutral hypersurface deforms towards the training data,  %in the case  (\ref{eq:vt1})
as if attracted by  a gravitational force due to
point masses centered at the training points. Simultaneously, the regularization term forces the hypersurface to remain as flat (or as volume minimizing) as possible, as if in the presence of surface tension.
 Thus this term follows the physics of  soap films and minimal surfaces~\cite{Dierkes1992}.
Geometric flows like the one proposed here are often modeled on physical processes. In our case, the flow can be viewed as a mixed gravity and surface tension physical experiment.

%(\emph{Qinxun: can we possibly make this argument softer, sounds more like an analogy/inspiration rather than a support/justification?})

%\begin{figure*}[htb]
%\centering
%\begin{tabular}{cc}
%\includegraphics[width=.43\linewidth]{pic/ionosphere_result.eps}&
%\includegraphics[width=.43\linewidth]{pic/glass_result.eps}
%%(a) Ionos. & (b) Glass \\
%
%%\includegraphics[width=.46\linewidth,height=1in]{picture/football1.eps}&
%%\includegraphics[width=.46\linewidth,height=1in]{picture/football2.eps}\\
%\vspace{-0.2cm}
%\end{tabular}
%\caption{Tenfold cross-validation training error (blue) and testing error (red) for The Ionosphere and Glass datasets as a function of the trade-off parameter $\lambda$. }
%\label{fig:lambda}
%\vspace{-0.2cm}
%\end{figure*}

\section{Conclusion}
\label{sec:future}

We have introduced a new geometric perspective on regularization for classification that exploits the geometry of a robust class probability estimator. Under this perspective, we propose a general regularization approach that applies to both binary and multiclass cases in a unified way.
In experiments with an example formulation based on RBFs, our implementation achieves favorable results comparing with widely used RBF-based classification methods and previous geometric regularization methods.
%We have proposed a geometric setup with a volume regularization term to uniformly solve binary and multiclass classification via gradient flow methods.
While experimental results demonstrate the effectiveness of our geometric regularization technique, it is also important to study convergence properties of this approach from a learning theory perspective. As an initial attempt, we have established Bayes consistency for an easy case of empirical penalty function and details are provided in Appendix~\ref{sec:Bayes}. We will continue this study in the future.

%Existing techniques~\cite{Steinwart2005,SteinwartScovel2007} for analyzing Bayes consistency and convergence rates of regularized ERM schemes involve bounding the capacity of the functional space related to the analytic regularizer via covering number estimates. However, a covering number bound for the functional space shrunk by our geometric regularizer seems hard in general.\footnote{A covering number estimate is obtained in~\citet{VarshneyWillsky2010}, where a volume regularizer is used for decision boundaries. However, their combinatorial techniques  only apply to their particular choice of functional space, i.e., sign distance functions.} This again indicates that our regularizer is essentially different from commonly used functional norms, since even the complicated Sobolev norm spaces  have covering number bounds~\cite{CuckerSmale2002}.
%{\color{blue} The supplemental material contains an example of a particular data term for which we can bypass covering number estimates and obtain Bayes consistency.}
%For future work, we will  investigate necessarily more sophisticated techniques for measuring the capacity of the functional space related to our geometric regularizer and analyze their convergence properties. {\color{blue} We will also evaluate our work on larger scale datasets.}

% In the unusual situation where you want a paper to appear in the
% references without citing it in the main text, use \nocite
%\nocite{langley00}

\bibliography{nipspaper}

\begin{thebibliography}{34}
\providecommand{\natexlab}[1]{#1}
\providecommand{\url}[1]{\texttt{#1}}
\expandafter\ifx\csname urlstyle\endcsname\relax
  \providecommand{\doi}[1]{doi: #1}\else
  \providecommand{\doi}{doi: \begingroup \urlstyle{rm}\Url}\fi

\bibitem[Audibert \& Tsybakov(2007)Audibert and Tsybakov]{Audibert2007}
Audibert, Jean-Yves and Tsybakov, Alexandre.
\newblock Fast learning rates for plug-in classifiers.
\newblock \emph{Annals of Statistics}, 35\penalty0 (2):\penalty0 608--633,
  2007.

\bibitem[Bartlett et~al.(2006)Bartlett, Jordan, and McAuliffe]{Bartlett2006}
Bartlett, Peter~L, Jordan, Michael~I, and McAuliffe, Jon~D.
\newblock Convexity, classification, and risk bounds.
\newblock \emph{Journal of the American Statistical Association}, 101\penalty0
  (473):\penalty0 138--156, 2006.

\bibitem[Belkin \& Niyogi(2003)Belkin and Niyogi]{BelkinNiyogi2003}
Belkin, Mikhail and Niyogi, Partha.
\newblock Laplacian eigenmaps for dimensionality reduction and data
  representation.
\newblock \emph{Neural Computation}, 15\penalty0 (6):\penalty0 1373--1396,
  2003.

\bibitem[Belkin et~al.(2006)Belkin, Niyogi, and Sindhwani]{Belkin2006}
Belkin, Mikhail, Niyogi, Partha, and Sindhwani, Vikas.
\newblock Manifold regularization: A geometric framework for learning from
  labeled and unlabeled examples.
\newblock \emph{Journal of Machine Learning Research}, 7:\penalty0 2399--2434,
  2006.

\bibitem[Cai \& Sowmya(2007)Cai and Sowmya]{CaiSowmya2007}
Cai, Xiongcai and Sowmya, Arcot.
\newblock Level learning set: A novel classifier based on active contour
  models.
\newblock In \emph{Proc.\ European Conf.\ on Machine Learning (ECML)}, pp.\
  79--90. 2007.

\bibitem[Chang \& Lin(2011)Chang and Lin]{ChangLin2011}
Chang, Chih-Chung and Lin, Chih-Jen.
\newblock {LIBSVM}: A library for support vector machines.
\newblock \emph{ACM Transactions on Intelligent Systems and Technology},
  2:\penalty0 27:1--27:27, 2011.
\newblock Software available at \url{http://www.csie.ntu.edu.tw/~cjlin/libsvm}.

\bibitem[Chen et~al.(1999)Chen, Giga, and Goto]{ChenGigaGoto}
Chen, Yun-Gang, Giga, Yoshikazu, and Goto, Shun'ichi.
\newblock Uniqueness and existence of viscosity solutions of generalized mean
  curvature flow equations.
\newblock In \emph{Fundamental contributions to the continuum theory of
  evolving phase interfaces in solids}, pp.\  375--412. Springer, Berlin, 1999.

\bibitem[Devroye et~al.(1996)Devroye, Gy{\"o}rfi, and Lugosi]{Devroye1996}
Devroye, Luc, Gy{\"o}rfi, L{\'a}szl{\'o}, and Lugosi, G{\'a}bor.
\newblock \emph{A probabilistic theory of pattern recognition}.
\newblock Springer, 1996.

\bibitem[Dierkes et~al.(1992)Dierkes, Hildebrandt, K{\"u}ster, and
  Wohlrab]{Dierkes1992}
Dierkes, Ulrich, Hildebrandt, Stefan, K{\"u}ster, Albrecht, and Wohlrab,
  Ortwin.
\newblock \emph{Minimal surfaces}.
\newblock Springer, 1992.

\bibitem[Donoho \& Grimes(2003)Donoho and Grimes]{DonohoGrimes2003}
Donoho, David and Grimes, Carrie.
\newblock Hessian eigenmaps: Locally linear embedding techniques for
  high-dimensional data.
\newblock \emph{Proceedings of the National Academy of Sciences}, 100\penalty0
  (10):\penalty0 5591--5596, 2003.

\bibitem[FMD()]{FMD}
FMD.
\newblock \url{http://people.csail.mit.edu/celiu/CVPR2010/FMD/}.
\newblock Accessed: 2015-06-01.

\bibitem[Goodfellow et~al.(2014)Goodfellow, Shlens, and
  Szegedy]{Goodfellow2014}
Goodfellow, Ian~J, Shlens, Jonathon, and Szegedy, Christian.
\newblock Explaining and harnessing adversarial examples.
\newblock \emph{arXiv preprint arXiv:1412.6572}, 2014.

\bibitem[Guckenheimer \& Worfolk(1993)Guckenheimer and
  Worfolk]{GuckenherimerWorfolk}
Guckenheimer, John and Worfolk, Patrick.
\newblock Dynamical systems: some computational problems.
\newblock In \emph{Bifurcations and periodic orbits of vector fields
  ({M}ontreal, {PQ}, 1992)}, volume 408 of \emph{NATO Adv. Sci. Inst. Ser. C
  Math. Phys. Sci.}, pp.\  241--277. Kluwer Acad. Publ., Dordrecht, 1993.

\bibitem[LeCun et~al.(1998)LeCun, Bottou, Bengio, and Haffner]{Lecun1998}
LeCun, Yann, Bottou, L{\'e}on, Bengio, Yoshua, and Haffner, Patrick.
\newblock Gradient-based learning applied to document recognition.
\newblock \emph{Proceedings of the IEEE}, 86\penalty0 (11):\penalty0
  2278--2324, 1998.

\bibitem[Lin \& Zha(2008)Lin and Zha]{LinZha2008}
Lin, Tong and Zha, Hongbin.
\newblock Riemannian manifold learning.
\newblock \emph{IEEE Trans.\ on Pattern Analysis and Machine Intelligence
  (PAMI)}, 30\penalty0 (5):\penalty0 796--809, 2008.

\bibitem[Lin et~al.(2012)Lin, Xue, Wang, and Zha]{Lin2012}
Lin, Tong, Xue, Hanlin, Wang, Ling, and Zha, Hongbin.
\newblock Total variation and {E}uler's elastica for supervised learning.
\newblock \emph{Proc.\ International Conf.\ on Machine Learning (ICML)}, 2012.

\bibitem[Lin et~al.(2015)Lin, Xue, Wang, Huang, and Zha]{Lin2015}
Lin, Tong, Xue, Hanlin, Wang, Ling, Huang, Bo, and Zha, Hongbin.
\newblock Supervised learning via euler's elastica models.
\newblock \emph{Journal of Machine Learning Research}, 16:\penalty0 3637--3686,
  2015.

\bibitem[Mantegazza(2011)]{Mantegazza}
Mantegazza, Carlo.
\newblock \emph{Lecture Notes on Mean Curvature Flow}, volume 290 of
  \emph{Progress in Mathematics}.
\newblock Birkh\"auser/Springer Basel AG, Basel, 2011.

\bibitem[MNIST()]{MNIST}
MNIST.
\newblock \url{http://http://yann.lecun.com/exdb/mnist/}.
\newblock Accessed: 2015-06-01.

\bibitem[Mumford \& Shah(1989)Mumford and Shah]{Mumford1989}
Mumford, David and Shah, Jayant.
\newblock Optimal approximations by piecewise smooth functions and associated
  variational problems.
\newblock \emph{Communications on pure and applied mathematics}, 42\penalty0
  (5):\penalty0 577--685, 1989.

\bibitem[Osher \& Sethian(1988)Osher and Sethian]{OsherSethian1988}
Osher, Stanley and Sethian, James.
\newblock Fronts propagating with curvature-dependent speed: algorithms based
  on hamilton-jacobi formulations.
\newblock \emph{Journal of Computational Physics}, 79\penalty0 (1):\penalty0
  12--49, 1988.

\bibitem[Roscher et~al.(2012)Roscher, F{\"o}rstner, and Waske]{Roscher2012}
Roscher, Ribana, F{\"o}rstner, Wolfgang, and Waske, Bj{\"o}rn.
\newblock I 2 vm: incremental import vector machines.
\newblock \emph{Image and Vision Computing}, 30\penalty0 (4):\penalty0
  263--278, 2012.

\bibitem[Roweis \& Saul(2000)Roweis and Saul]{RoweisSaul2000}
Roweis, Sam and Saul, Lawrence.
\newblock Nonlinear dimensionality reduction by locally linear embedding.
\newblock \emph{Science}, 290\penalty0 (5500):\penalty0 2323--2326, 2000.

\bibitem[Sch{\"o}lkopf \& Smola(2002)Sch{\"o}lkopf and
  Smola]{ScholkopfSmola2002}
Sch{\"o}lkopf, Bernhard and Smola, Alexander.
\newblock \emph{Learning with kernels: Support vector machines, regularization,
  optimization, and beyond}.
\newblock MIT press, 2002.

\bibitem[Sethian(1999)]{Sethian1999}
Sethian, James~Albert.
\newblock \emph{Level set methods and fast marching methods: evolving
  interfaces in computational geometry, fluid mechanics, computer vision, and
  materials science}, volume~3.
\newblock Cambridge university press, 1999.

\bibitem[Steinwart(2005)]{Steinwart2005}
Steinwart, Ingo.
\newblock Consistency of support vector machines and other regularized kernel
  classifiers.
\newblock \emph{IEEE Trans.\ Information Theory}, 51\penalty0 (1):\penalty0
  128--142, 2005.

\bibitem[Stone(1977)]{Stone1977}
Stone, Charles.
\newblock Consistent nonparametric regression.
\newblock \emph{Annals of Statistics}, pp.\  595--620, 1977.

\bibitem[Tenenbaum et~al.(2000)Tenenbaum, De~Silva, and
  Langford]{Tenenbaum2000}
Tenenbaum, Joshua, De~Silva, Vin, and Langford, John.
\newblock A global geometric framework for nonlinear dimensionality reduction.
\newblock \emph{Science}, 290\penalty0 (5500):\penalty0 2319--2323, 2000.

\bibitem[Vapnik(1998)]{Vapnik1998}
Vapnik, Vladimir~Naumovich.
\newblock \emph{Statistical learning theory}, volume~1.
\newblock Wiley New York, 1998.

\bibitem[Varshney \& Willsky(2010)Varshney and Willsky]{VarshneyWillsky2010}
Varshney, Kush and Willsky, Alan.
\newblock Classification using geometric level sets.
\newblock \emph{Journal of Machine Learning Research}, 11:\penalty0 491--516,
  2010.

\bibitem[VLFeat()]{VLFeat}
VLFeat.
\newblock \url{http://www.vlfeat.org/applications/apps.html}.
\newblock Accessed: 2015-06-01.

\bibitem[Zhang \& Zha(2005)Zhang and Zha]{ZhangZha2005}
Zhang, Zhenyue and Zha, Hongyuan.
\newblock Principal manifolds and nonlinear dimensionality reduction via
  tangent space alignment.
\newblock \emph{SIAM Journal on Scientific Computing}, 26\penalty0
  (1):\penalty0 313--338, 2005.

\bibitem[Zhou \& Sch{\"o}lkopf(2005)Zhou and Sch{\"o}lkopf]{Zhou2005}
Zhou, Dengyong and Sch{\"o}lkopf, Bernhard.
\newblock Regularization on discrete spaces.
\newblock In \emph{Pattern Recognition}, pp.\  361--368. Springer, 2005.

\bibitem[Zhu \& Hastie(2005)Zhu and Hastie]{Zhu2005}
Zhu, Ji and Hastie, Trevor.
\newblock Kernel logistic regression and the import vector machine.
\newblock \emph{Journal of Computational and Graphical Statistics}, 2005.

\end{thebibliography}
\bibliographystyle{icml2016}

\clearpage

\appendix

\section{Proof of Theorem~\ref{thm:multiclass}}
\label{sec:multi_proof}

\begin{proof}
For $\vec f:\R^N\to \Delta^{L-1}\subset\R^L$,
$$\{r_j = r_j(\vec x)= (0,\ldots,\overset{j}{1},\ldots,0, f^1_j,\ldots, f^L_j): j = 1,\ldots N\}$$
 is a basis of the tangent space $T_{\vec x}\grf$ to
$\grf$.  Here $f^i_j = \partial_{x^j}f^i.$ Let $\{e_i\}$ be an orthonormal frame of $T_{\vec x}\grf.$
We have
$$e_i = B_i^jr_j$$
for some invertible matrix $B_i^j.$

Define the metric matrix $g$ for the basis $\{r_j\}$ by
$$g = (g_{kj})\ {\rm with}\ g_{kj} = r_k\cdot r_j = \delta_{kj} + f^i_kf^i_j.$$
Then
\begin{eqnarray*}
\delta_{ij} = e_i\cdot e_j = B_i^k B_j^t r_k\cdot r_t=  B_i^k B_j^t  g_{kt}\\
\Rightarrow I = (BB^T) g\Rightarrow BB^T = g^{-1}.
\end{eqnarray*}
Thus $BB^T$ is computable in terms of derivatives of $\vec f$.

Let $D_uw$ be the $\R^{N+L}$ directional derivative of $w$ in the direction $u$.  Then
\begin{eqnarray*}
\trii &=& P^\nu D_{e_i}e_i = P^\nu D_{B_i^j r_j}B_i^k r_k = B_i^j P^\nu D_{ r_j}B_i^k r_k\\
&=& B_i^j P^\nu[  (D_{r_j} B_i^k)r_k] + B_i^jB_i^k D_{r_j}r_k \\
&=& B_i^jB_i^k P^\nu D_{r_j}r_k\\
&=& (g^{-1})^{jk}P^\nu D_{r_j}r_k,
\end{eqnarray*}
since $P^\nu r_k = 0.$

We have
\begin{eqnarray*}
r_k &=& (0,\ldots, 1,\ldots, f^1_k(x^1,\ldots, x^N),\ldots, f^L_k(x^1,\ldots, x^N))\\
 &=& \partial^{\R^{N+L}}_k + \sum_{i=1}^L f^i_{k}\partial^{\R^{N+L}}_{N+i},
\end{eqnarray*}
so in particular, $\partial^{\R^{N+L}}_\ell r_k = 0$ if $\ell>N.$
Thus
$$D_{r_j}r_k = (0,\ldots, \overset{N}{0}, f^1_{kj},\ldots, f^L_{kj}).$$

So far, we have
$$\trii  = (g^{-1})^{jk}P^\nu(0,\ldots, \overset{N}{0}, f^1_{kj},\ldots, f^L_{kj}).$$

Since $g$ is given in terms of derivatives of $\vec f$, we need to write $P^\nu = I - P^T$ in terms of derivatives of $\vec f$.
For any $u\in \R^{N+L}$, we have
\begin{eqnarray*}
P^T u &=& (P^T u\cdot e_i) e_i = (u\cdot B_i^j r_j) B_i^k r_k \\
&=& B_i^jB_i^k (u\cdot r_j) r_k\\
&=& (g^{-1})^{jk}(u\cdot r_j) r_k.
\end{eqnarray*}

Thus
\begin{eqnarray}\label{one} &&\trii \\
&=& (g^{-1})^{jk}P^\nu(0,\ldots, \overset{N}{0}, f^1_{kj},\ldots, f^L_{kj})\\
&=&  (g^{-1})^{jk}(0,\ldots, \overset{N}{0}, f^1_{kj},\ldots, f^L_{kj}) \nonumber\\
 &&- P^T[
(g^{-1})^{jk}(0,\ldots, \overset{N}{0}, f^1_{kj},\ldots, f^L_{kj}) \nonumber]\\
&=& (g^{-1})^{jk}(0,\ldots, \overset{N}{0}, f^1_{kj},\ldots, f^L_{kj}) \nonumber\\
 &&-
(g^{-1})^{jk}[ (g^{-1})^{rs}(0,\ldots, \overset{N}{0}, f^1_{rs},\ldots, f^L_{rs}) \cdot r_j] r_k \nonumber\\
&=&  (g^{-1})^{jk}(0,\ldots, \overset{N}{0}, f^1_{kj},\ldots, f^L_{kj}) \nonumber\\
 &&- (g^{-1})^{jk} (g^{-1})^{rs}  \left(f^i_{rs}f^i_j\right)  r_k \nonumber\\
 &=& (g^{-1})^{ij}
\biggl(0,\ldots,\overset{j}{ -(g^{-1})^{rs}  f^a_{rs}f^a_i },\ldots, 0, \\
 &&f^1_{ji} - (g^{-1})^{rs}  f^a_{rs}f^a_i f^1_j,
 \ldots, f^L_{ji} - (g^{-1})^{rs}  f^a_{rs}f^a_i f^L_j\biggr),\nonumber
\end{eqnarray}
after a relabeling of indices. Therefore, the last $L$ component of $\trii$ are given by
\begin{eqnarray*}
\trii^L
&=& (g^{-1})^{ij}
\biggl( f^1_{ji} - (g^{-1})^{rs}  f^a_{rs}f^a_i f^1_j,\ldots,\\
 &&f^L_{ji} - (g^{-1})^{rs}  f^a_{rs}f^a_i f^L_j\biggr).
\end{eqnarray*}

\end{proof}

\section{An Easy Example with Bayes Consistency}
\label{sec:Bayes}

We now give an example with a loss function that enables easy Bayes consistency proof under some mild initialization assumption. Related notation is summerized in~\S\ref{sec:notation}.

For ease of reading, we change the notation for empirical penalty $\calP_{\calT_m}$ in the Appendix to $\pdist$, i.e., $\calP = \pdist + \lambda\pg.$
$\pdist$ measures the deviation of $\grf$ from the mapped training points, a natural geometric distance penalty term is an $L^2$ distance in $\R^L$ from $\vec f(\vec x)$ to the averaged $\vec z$ component of the $k$-nearest training points:
\begin{equation}
\label{eq:pdist}
\pdist(\vec f)= R_{D, \calT_m, k}(\vec f)=\int_{\mathcal X} d^2\left(\vec f(\vec x), \frac{1}{k}\sum_{i=1}^k\tilde{\vec z}_i\right)d\vec x,
\end{equation}
where $d$ is the Euclidean distance in $\R^L$, $\tilde{\vec z}_i$ is the vector of the last $L$ components
of $(\tilde{\vec x}_i,\tilde{\vec z}_i)=(\tilde{\vec x}_i^1,\ldots,\tilde{\vec x}_i^N,\tilde{\vec z}_i^1,\ldots,\tilde{\vec z}_i^L)$, with $\tilde{\vec x}_i$ the $i^{\rm th}$ nearest neighbor of $\vec x$ in $\calT_m$, and $d\vec x$ is the Lebesgue measure. The gradient vector field is
\begin{equation*}
\nabla (R_{D, \calT_m, k})_{\vec f}(\vec x,\vec f(\vec x)) = \frac{2}{k}\sum_{i=1}^k  (\vec f(\vec x) - \tilde{\vec z}_i).
\end{equation*}
However, $\nabla (R_{D, \calT_m, k})_{\vec f}$ is discontinuous on the set $\calD$ of  points $\vec x$ such that $\vec x$ has equidistant training points among its $k$ nearest neighbors. $\calD$ is the union of $(N-1)$-dimensional
hyperplanes in $\calX$, so $\calD$ has measure zero.
Such points will necessarily exist unless  the last $L$ components
 of the mapped training points are all $1$ or all $0$.
To rectify this, we can smooth out $\nabla (R_{D, \calT_m, k})_{\vec f}$ to a vector field
\begin{equation}
\label{eq:vdist}
V_{D,\vec f,\phi} =    \frac{2\phi(\vec x)}{k}\sum_{i=1}^k  (\vec f(\vec x) - \tilde{\vec z}_i).
\end{equation}
Here $\phi(\vec x)$ is a smooth damping function close to the singular function $\delta_{\calD}$, which has $\delta_{\calD}(\vec x) = 0$ for $\vec x\in \calD$ and $\delta_{\calD}(\vec x) = 1$ for $\vec x\not\in \calD$.
Outside any open neighborhood of $\calD$,  $\nabla R_{D, \calT_m, k}=
 V_{D,\vec f,\phi}$ for $\phi$ close enough to $\delta_{\calD}.$
%$V_{D,\vec f,\phi}$ is no longer a gradient vector field, but this does not affect the theory in \S4 or the computations in \S5.

Recall the geometric penalty from the submission, i.e., $\pg(\vec f) = \int_{\grf}\dvol,$ with the geometric gradient vector field being $V_{G,\vec f}  = -\trii^L.$

Then the gradient vector field $V_{tot, \lambda, m, \vec f, \phi}$ of this example penalty $\calP$ is,
\begin{eqnarray}
\label{eq:vtot}
V_{tot, \lambda, m, \vec f, \phi} &=& \nabla\calP_f =  V_{D,\vec f,\phi} + \lambda V_{G,\vec f} \nonumber \\
&=&
\frac{2\phi(\vec x)}{k}\sum_{i=1}^k  (\vec f(\vec x) - \tilde{\vec z}_i)
-  \lambda\trii^L.
 \end{eqnarray}

\subsection{Consistency analysis}
\label{sec:slt}

For a training set $\calT_m$,
we let $\vec f_{\calT_m}=(f^1_{\calT_m},\ldots,f^L_{\calT_m})$ be the class probability estimator given by our approach.
We denote the generalization risk of the corresponding plug-in classifier $h_{\vec f_{\calT_m}}$ by $R_P(\vec f_{\calT_m})=\mathbb E_P[\bbo_{h_{\vec f_{\calT_m}}(\vec x)\neq y}]$.
The Bayes risk is defined by $R_P^*=\inf\limits_{h:\mathcal X\to\mathcal Y}R_P(h)=\mathbb{E}_P[\bbo_{h_{\vec \eta}(\vec x)\neq y}]$.
Our algorithm is Bayes consistent if $\lim\limits_{m\to\infty}R_P(\vec f_{\calT_m})=R^*_P$ holds in probability for all distributions $P$ on $\mathcal X\times\mathcal Y$.
Usually, gradient flow methods are applied to a convex functional, so that a flow line approaches the unique global minimum.  If the domain of the functional is an infinite dimensional manifold of (e.g. smooth) functions,
we always assume that flow lines exist and that the actual minimum exists in this manifold.

Because our functionals are not convex, and because we are strictly speaking not working with gradient vector fields, we can only hope to prove Bayes consistency for the set of
 initial estimators in the stable manifold of a stable fixed point (or sink) of the vector field~\cite{GuckenherimerWorfolk}.
Recall that a stable fixed  point $\vec f_0$ has a maximal open neighborhood,
the stable manifold $\calS_{\vec f_0}$, on which flow lines tend towards $\vec f_0$.
For the manifold $\calM$, the stable manifold for a stable critical point of the vector field $V_{tot,\lambda, m, \vec f,\phi}$ is infinite dimensional.

The proof of Bayes consistency for multiclass (including binary) classification follows these steps:
\medskip

\noindent {\bf Step 1:}
$\lim\limits_{\lambda\to 0} R^*_{D,P,\lambda} %(f^C_{D,P,\lambda})
=0.$
%R^C_{D,P}.$

\noindent {\bf Step 2:} $\lim\limits_{n\to\infty}R_{D,P}(\vec f_n)= 0%R^C_{D,P}
  \Rightarrow\lim\limits_{n\to\infty}R_P(\vec f_n)=R^*_P$.

\noindent {\bf Step 3:}  For all $\vec f\in \calM = \maps(\calX, \Delta^{L-1})$,
$|R_{D,\mathcal T_m}(\vec f)-R_{D,P}(\vec f)|\xrightarrow{m\to\infty} 0$ in probability.
\medskip

%As explained in \S5.2, in the binary case we can take $f$ to have values in $[0,1].$
Proofs of these steps are provided in following subsections. For the notation see~\S\ref{sec:notation}. $R^*_{D,P, \lambda}$ is the minimum of the regularized $D$ risk
 $R_{D,P,\lambda}(\vec f)$
 for $\vec f$:  $R_{D,P,\lambda}(\vec f)=R_{D,P}(\vec f)+\lambda\pg(\vec f)$, with $R_{D,P}(\vec f)=\int_{\mathcal X} d^2(\vec f(\vec x), \vec\eta(\vec x))
d\vec x$ the $D$-risk. Also, $R_{D,\calT_m,\lambda}(\vec f)=R_{D,\calT_m}(\vec f)+\lambda\pg(\vec f)$, with
$R_{D,\calT_m}(\vec f) = \int_{\mathcal X} d^2\left(\vec f(\vec x), \frac{1}{k}\sum_{i=1}^k\tilde{\vec z}_i\right)d\vec x$ the empirical $D$-risk.
\medskip

\begin{theorem}[Bayes Consistency]
\label{thm:consistency}
%\begin{theorem}
Let $m$ be the size of the training data set. Let $\vec f_{1, \lambda,m}\in{\mathcal S}_{\vec f_{D,\calT_m,\lambda}}$, the stable manifold for
the global minimum  $\vec f_{D,\calT_m,\lambda}$ of %the vector field $V_{tot,\lambda,m, f}$
$R_{ D,\calT_m,\lambda}$, and let $\vec f_{n,\lambda,m,\phi}$
be a sequence of functions on the flow line of $V_{tot,\lambda,m,\vec f,\phi}$ starting with $\vec f_{1,\lambda, m}$ with the flow time $t_n\to\infty$ as $n\to \infty.$  Then
$R_P(\vec f_{n,\lambda,m,\phi})\xrightarrow[\lambda\to 0,\phi\to \delta_{\calD}]{m,n\to\infty}R^*_P$ in probability for all distributions $P$ on $\mathcal X\times\mathcal Y$, if $k/m\to 0$ as $m\to\infty$.
%\end{theorem}
\end{theorem}
%\medskip

\begin{proof}
In the notation of~\S\ref{sec:notation}, if $\vec f_{D,\calT_m,\lambda}$ is a global
  minimum for $R_{D,\calT_m,\lambda}$, then outside of $\calD$,
  $\vec f_{D,\calT_m,\lambda}$ is the limit of critical points for the negative flow of
  $V_{tot, \lambda, m, \vec f, \phi}$ as $\phi\to \delta_{\calD}.$
  To see this, fix an
  $\epsilon_i$ neighborhood $\calD_{\epsilon_i}$ of $\calD$.  For a sequence $\phi_j\to\delta_{\calD}, $
$V_{tot, \lambda, m, f, \phi_j}$ is     independent of $j \geq j(\epsilon_i)$ on $\calX\setminus \calD_{\epsilon_i}$,
%is equicontinuous on $\calX\setminus \calD_{\epsilon_i}$,
 so we find a function $\vec f_{i}$, a  critical point of $V_{tot, \lambda, m, \vec f, \phi_{j(\epsilon_i)}}$,
 equal to
$\vec f_{D,\calT_m,\lambda}$ on $\calX\setminus \calD_{\epsilon_i}$.  Since any $\vec x\not\in \calD$ lies outside some $\calD_{\epsilon_i}$, the sequence $\vec f_{i}$ converges at $\vec x$ if we let
$\epsilon_i\to 0.$  Thus we can ignore the choice of $\phi$ in our proof, and drop $\phi$ from the notation.

For our algorithm, for fixed $\lambda,m$, we have as above
$\lim\limits_{n\to\infty}\vec f_{n,\lambda, m} = \vec f_{D, \calT_m,\lambda}$, so
$$\lim\limits_{n\to\infty} R_{D, \calT_m, \lambda}(\vec f_{n,\lambda,m}) = R_{D, \calT_m, \lambda}(\vec f_{D, \calT_m,\lambda}),$$
for $\vec f_1\in {\mathcal S}_{\vec f_{D, \calT_m,\lambda}}.$
By Step 2, it suffices
to show $R_{D,P}(\vec f_{D,\mathcal T_m,\lambda})\xrightarrow[\lambda\to 0]{m\to\infty}0$.%R^{C_0}_{D,P}$,
 %where $R_{D,P}^{C_0}$  is the global minimum for the $D$-risk.
\ In probability, we have
$\forall\delta>0,\exists m>0$ such that
\begin{eqnarray*} \lefteqn{
R_{D,P}(\vec f_{D,\mathcal T_m,\lambda})}\\
&\leq&
R_{D,P}(\vec f_{D,\mathcal T_m,\lambda}) + \lambda\pg(\vec f_{D,\mathcal T_m,\lambda})\\
&\leq& R_{D,\mathcal T_m}(\vec f_{D,\mathcal T_m,\lambda}) + \lambda\pg(\vec f_{D,\mathcal T_m,\lambda}) + \frac{\delta}{3} \ \ \text{ (Step 3)}\\
&=& R_{D,\calT_m,\lambda}(\vec f_{D,\mathcal T_m,\lambda}) + \frac{\delta}{3} \\
&\leq& R_{D,\calT_m,\lambda}(\vec f_{D,P,\lambda})
+ \frac{\delta}{3}  \ \ \text{ (minimality of $\vec f_{D,\mathcal T_m,\lambda}$)} \\
&=& R_{D,\mathcal T_m}(\vec f_{D,P,\lambda}) + \lambda\pg(\vec f_{D,P,\lambda}) + \frac{\delta}{3} \\
&\leq& R_{D,P}(\vec f_{D,P,\lambda}) + \lambda\pg(\vec f_{D,P,\lambda}) + \frac{2\delta}{3} \text{ (Step
3)} \\
&=& R_{D,P,\lambda}(\vec f_{D,P,\lambda}) +\frac{2\delta}{3}  = R^*_{D,P,\lambda} +\frac{2\delta}{3} \\
%\leq R_{D,P}(\eta) + \lambda\Omega(\eta) +\frac{2\delta}{3} \\
%& = & \lambda\Omega(\eta) +\frac{2\delta}{3}
&\leq& \delta, \ \ \text{ (Step 1)}
\end{eqnarray*}
for $\lambda$ close to zero.
%for
%$\lambda \leq\frac{\delta}{3\Omega(\eta)}.$
Since $R_{D,P}(\vec f_{D,\mathcal T_m,\lambda})\geq 0$, we are done.
%This estimate also shows that every $f_1\in {\mathcal S}_f$ has a sequence
%$f_{1,\lambda}\in {\mathcal S}_{f^C_{D, \calT_m,\lambda}}$ with $f_{1,\lambda}\xrightarrow[\lambda\to 0]{} f_1$ pointwise.
\end{proof}

\subsection{Step 1}

 \begin{lemma} (Step 1) $\lim\limits_{\lambda\to 0}R^*_{D,P,\lambda}%(f^C_{D,P,lambda})
  = 0.$ %R^C_{D,P}$.
 \end{lemma}

 \begin{proof}  After the smoothing procedure in \S3.1 for the distance penalty term, the function
 $R_{D,P,\lambda}:\calM
 \to\R$ is continuous in the Fr\'echet topology on $\calM.$
%on $\maps(\calX,\Delta^{K-1})$, the space of smooth maps from $\calX$ to $\Delta^{K-1}$.
%\footnote{I have to check this for the distance penalty function.}
 We check that the functions $R_{D,P,\lambda}:\calM\to\R$ are equicontinuous
 in $\lambda$: for fixed $\vec f_0\in \calM$ and $\epsilon > 0$, there exists $\delta = \delta(\vec f_0, \epsilon)$
 such that
 $|\lambda -\lambda'| <\delta\Rightarrow |R_{D,P,\lambda}(\vec f_0) - R_{D,P,\lambda'}(\vec f_0)|
 <\epsilon.$
 This is immediate:
 $$|R_{D,P,\lambda}(\vec f_0) - R_{D,P,\lambda'}(\vec f_0)| =
 |(\lambda -\lambda')\pg(\vec f_0)| < \epsilon,$$
 if $\delta < \epsilon/|\pg(\vec f_0)|.$
 It is standard that the infimum $\inf R_\lambda$ of an equicontinuous family of functions is continuous in
 $\lambda$, so
 $\lim\limits_{\lambda\to 0} R^*_{D,P,\lambda} = R^*_{D,P,\lambda = 0}
 = R_{D,P}(\vec\eta)= 0$.  \end{proof}

\subsection{Step 2}

%We  have to decide what we mean by $\maps(\calX,\R)$ -- I suppose we mean smooth maps, so we can use flow methods.  In this case, we also {\it assume} that $\pxy$ is smooth.  We need to clarify whether the Bayes classifier is a min or an inf.

We assume that the class probability function $\vec\eta(\vec x):\R^N\to\R^L$ is smooth, and that the marginal distribution $\mu(\vec x)$ is continuous.  We also let $\mu$ denote the corresponding measure
on $\calX.$
\medskip

\noindent {\bf Notation:}
$$h_{\vec f}(\vec x) = \argmax\{f^{\ell}(\vec x),\ell\in\mathcal Y\}.$$
Of course,
$$\bbo_{h_{\vec f}(\vec x)\neq y} = \left\{ \begin{array}{rr} 1,& h_{\vec f}(\vec x)\neq y,\\
0,& h_{\vec f}(\vec x) = y.\end{array}\right.$$

\begin{lemma} (Step 2 for %$C_0$ on
a subsequence)
$$\lim\limits_{n\to\infty}R_{D,P}(\vec f_n)=0%R^{C_0}_{D,P}
 \Rightarrow\lim\limits_{i\to\infty}R_P(\vec f_{n_i})=R^*_P$$ for
some subsequence $\{\vec f_{n_i}\}_{i=1}^\infty$ of $\{\vec f_n\}.$
\end{lemma}

\begin{proof}
%  Since $\eta(x) = \pxy$ is smooth, $R_{D,P}(\pxy) = 0$, and so $R^{C_0}_{D,P} = %R^*_P =
%   0.$
%   Thus we want
%  \begin{equation}\label{3.0}\lim\limits_{n\to\infty}R_{D,P}(f_n)=0\Rightarrow\lim\limits_{n\to\infty}R_p(f_n)=R^*_P.
%  \end{equation}
The left hand side of the Lemma is
$$\int_{\mathcal X} d^2(\vec f_n(\vec x), \vec\eta(\vec x))d\vec x
%= \int_\calX d_{\R}^2(f_n(\vec x), \eta(\vec x))\mx
 \to 0,$$
which is equivalent to
\begin{equation}\label{3.0a}\int_{\mathcal X} d^2(\vec f_n(\vec x), \vec\eta(\vec x))\mx
 \to 0,
 \end{equation}
 since $\calX$ is compact and $\mu$ is continuous.
Therefore,
it suffices to show
\begin{eqnarray}\label{3.3}\lefteqn{\int_\calX d^2(\vec f_n(\vec x), \vec\eta(\vec x))\mx \to 0}\\
&\Longrightarrow &
%\int_\calX\mon dx \to \int_\calX \mpp dx.\nonumber
{\mathbb E}_P[\bbo_{h_{\vec f_n}(\vec x)\neq y} ]
\to {\mathbb E}_P[\bbo_{h_{\vec\eta}(\vec x)\neq y} ].\nonumber
\end{eqnarray}

We recall that $L^2$ convergence implies pointwise convergence a.e, so (\ref{3.0a}) implies that a subsequence of $\vec f_n$, also denoted $\vec f_n$, has
$\vec f_n\to \vec\eta(\vec x)$ pointwise a.e. on $\calX$.  (By our assumption on $\mu(\vec x)$, these statements hold for
either $\mu$ or Lebesgue measure.) By Egorov's theorem, for any $\epsilon >0$, there exists a set
$B_\epsilon \subset \calX$ with $\mu(B_\epsilon)  < \epsilon$ such that $\vec f_n\to \vec\eta(\vec x)$ uniformly on $\calX\setminus
B_\epsilon.$

Fix $\delta>0$ and set
\begin{eqnarray*}
Z_\delta &=& \{\vec x\in \calX: \#\{\underset{\ell\in\mathcal Y}{\rm argmax}\ \eta^\ell(\vec x)\} = 1, \\
&&|\max\limits_{\ell\in\mathcal Y}\eta^\ell(\vec x) -
\underset{\ell\in\mathcal Y}\smax\ \eta^\ell(\vec x)|
< \delta\},
\end{eqnarray*}
where $\underset{\ell\in\mathcal Y}\smax$
 denotes the second largest element in $\{\eta^1(\vec x),\ldots,\eta^L(\vec x)\}$.
For the moment, assume that
$ Z_0 =  \{\vec x\in \calX: \#\{\underset{\ell\in\mathcal Y}{\rm argmax}\ \eta^\ell(\vec x)\} > 1\}$
%\eta^\ell(\vec x) = 1/L, \forall \ell\in {\mathcal Y}\}$
has $\mu(Z_0) = 0$.

It follows easily\footnote{Let $A_k$ be sets with $A_{k+1}\subset A_k$ and with $\mu(\cap_{k=1}^\infty A_k) = 0.$  If $\mu(A_k)\not\to 0$, then there exists a subsequence, also called $A_k$, with $\mu(A_k) > K > 0$ for some $K$.  We claim
$\mu(\cap A_k) \geq K$, a contradiction.  For the claim, let $Z = \cap A_k$. If $\mu(Z) \geq \mu(A_k)$ for all $k$, we are done.  If not, since the $A_k$ are nested, we can replace $A_k$ by a set, also called $A_k$, of
measure $K$ and such that the new $A_k$ are still nested.  For the relabeled $Z = \cap A_k$,
$Z\subset A_k$ for all $k$, and we may assume $\mu(Z) < K.$   Thus there exists $Z'\subset A_1$ with $Z'\cap Z = \emptyset$ and
$\mu(Z') > 0.$  Since $\mu(A_i) =K$, we must have $A_i\cap Z'\neq \emptyset$ for all $i$.  Thus
$\cap A_i$ is strictly larger than $Z$, a contradiction.  In summary, the claim must hold, so we get a contradiction to assuming $\mu(A_k)\not\to 0$.}  that $\mu(Z_\delta) \to 0$ as $\delta
\to 0.$  On $\calX\setminus (Z_\delta\cup B_\epsilon)$, we have
$\bbo_{h_{\vec f_n}(\vec x)\neq y} = \bbo_{h_{\vec\eta}(\vec x)\neq y}$ for $n> N_\delta$.
Thus
$${\mathbb E}_P[\bbo_{\calX\setminus (Z_\delta\cup B_\epsilon)}
\bbo_{h_{\vec f_n}(\vec x)\neq y} ]
= {\mathbb E}_P[\bbo_{\calX\setminus (Z_\delta\cup B_\epsilon)}
\bbo_{h_{\vec\eta}(\vec x)\neq y} ].$$
(Here $\bbo_A$ is the characteristic function of a set $A$.)

As $\delta\to 0$,
$${\mathbb E}_P[\bbo_{\calX\setminus (Z_\delta\cup B_\epsilon)}\bbo_{h_{\vec f_n}(\vec x)\neq y} ]
\to
{\mathbb E}_P[\bbo_{\calX\setminus B_\epsilon}\bbo_{h_{\vec f_n}(\vec x)\neq y} ].$$
and similarly for $\vec f_n$ replaced by $\vec\eta(\vec x).$
 During this process, $N_\delta$ presumably goes to $\infty$, but that
precisely means
$$\lim_{n\to\infty} {\mathbb E}_P[\bbo_{_{\calX\setminus B_\epsilon}}\bbo_{h_{\vec f_n}(\vec x)\neq y} ]
= {\mathbb E}_P[\bbo_{_{\calX\setminus B_\epsilon}}\bbo_{h_{\vec\eta}(\vec x)\neq y} ].$$

Since
$$\left| {\mathbb E}_P[\bbo_{\calX\setminus B_\epsilon} \bbo_{h_{\vec f_n}(\vec x)\neq y}] -
{\mathbb E}_P[ \bbo_{h_{\vec f_n}(\vec x)\neq y} ] \right| < \epsilon,$$
and similarly for $\vec\eta(\vec x)$, we get
\begin{eqnarray*}\lefteqn{ \left| \lim_{n\to\infty} {\mathbb E}_P[ \bbo_{h_{\vec f_n}(\vec x)\neq y} ]
-{\mathbb E}_P[ \bbo_{h_{\vec\eta}(\vec x)\neq y} ]\right| } \\
&\leq & \left| \lim_{n\to\infty} {\mathbb E}_P[\bbo_{h_{\vec f_n}(\vec x)\neq y}] -
\lim_{n\to\infty} {\mathbb E}_P[\bbo_{\calX\setminus B_\epsilon} \bbo_{h_{\vec f_n}(\vec x)\neq y}]\right|\\
&&\quad
+ \left| \lim_{n\to\infty}{\mathbb E}_P[ \bbo_{\calX\setminus B_\epsilon} \bbo_{h_{\vec f_n}(\vec x)\neq y}] -
{\mathbb E}_P[ \bbo_{\calX\setminus B_\epsilon} \bbo_{h_{\vec\eta}(\vec x)\neq y} ]\right| \\
&& \quad + \left| \lim_{n\to\infty} {\mathbb E}_P[\bbo_{\calX\setminus B_\epsilon} \bbo_{h_{\vec\eta}(\vec x)\neq y}]
- {\mathbb E}_P[  \bbo_{h_{\vec\eta}(\vec x)\neq y}]\right| \\
&\leq & 3\epsilon.
\end{eqnarray*}
(Strictly speaking, $ \lim_{n\to\infty} {\mathbb E}_P[ \bbo_{h_{\vec f_n}(\vec x)\neq y} ]$
is first $\limsup$ and then $\liminf$ to show that the  limit exists.)
Since $\epsilon$ is arbitrary, the proof is complete if $\mu(Z_0) = 0.$

If $\mu(Z_0) > 0$, we rerun the proof with $\calX$ replaced by $Z_0.$  As above, $\vec f_n|_{Z_0}$ converges
 uniformly to $\vec\eta(\vec x)$ off a set of measure $\epsilon$.  The argument above, without the set $Z_\delta$, gives
 $$\int_{Z_0} \mon \mx\to\int_{Z_0}\mpp \mx.$$
   We then proceed with the proof above on $\calX \setminus Z_0.$
\end{proof}

\begin{corollary} (Step 2 in general)  For our algorithm,
$\lim\limits_{n\to\infty}R_{D,P}(\vec f_{n,\lambda ,m})=0 %R^{C}_{D,P}
 \Rightarrow\lim\limits_{i\to\infty}R_P(\vec f_{n,\lambda,m})=  R_P^*.$ % R^{C}_P$.
%where $f_n\in \maps_{C_1}(\calX,\R).$
\end{corollary}

\begin{proof}
Choose $\vec f_{1,\lambda, m}$ as in Theorem 2. Since $V_{tot, \lambda,m, \vec f_{n,\lambda,m}}$ has pointwise length going to zero as $n\to\infty$,   $\{\vec f_{n,\lambda,m}(\vec x)\}$  is a Cauchy sequence for all $\vec x$.
This implies that $\vec f_{n,\lambda,m}$, and not just a subsequence, converges pointwise to $\vec\eta.$
%$f^{C%_1}_{D, \calT_M}$, the substitute for $\eta(x)$ in the previous lemma.
 %The rest of the proof proceeds as before.
%For a generic initial function $f_1$, the flow line of a gradient flow tends towards a local minimum of the distance penalty function.
\end{proof}

\subsection{Step 3}
\begin{lemma} (Step 3)
If $k\to\infty$ and $k/m\to 0$ as $m\to\infty$, then for $\vec f\in \maps(\calX, \Delta^{L-1})$,
$$|R_{D,\mathcal T_m}(\vec f)-R_{D,P}(\vec f)|\xrightarrow{m\to\infty} 0\ \ \text{in probability},$$
 for all distributions $P$ that generate $\mathcal T_m$.
\end{lemma}

\begin{proof}
Since $R_{D,P}(\vec f)$ is a constant for fixed $\vec f$ and $P$, convergence in probability
%is equivalent to
will follow from weak convergence, i.e.,
$${\mathbb E}_{\mathcal T_m}[|R_{D,\mathcal T_m}(\vec f)-R_{D,P}(\vec f)|]\xrightarrow{m\to\infty} 0.$$
We have
\begin{eqnarray*}
&&|R_{D,\mathcal T_m}(\vec f)-R_{D,P}(\vec f)| \\
&=&\left|\int_{\mathcal X} \left[d^2\left(\vec f(\vec x), \frac{1}{k}\sum_{i=1}^k\tilde{\vec z}_i\right) - d^2(\vec f(\vec x), \vec\eta(\vec x))\right]d\vec x\right|\\
&\leq&\int_{\mathcal X}\left|d^2\left(\vec f(\vec x),  \frac{1}{k}\sum_{i=1}^k\tilde{\vec z}_i\right) - d^2(\vec f(\vec x), \vec\eta(\vec x))\right| d\vec x.
\end{eqnarray*}
Set $\vec a = \vec f(\vec x) - \frac{1}{k}\sum_{i=1}^k\tilde{\vec z}_i,\ \vec b = \vec f(\vec x) - \vec\eta(\vec x).$  Then
\begin{eqnarray*}
&&\left|\|\vec a\|_2^2-\|\vec b\|_2^2\right| \\
&=& \left|\sum_{\ell=1}^La_{\ell}^2 - \sum_{\ell=1}^Lb_{\ell}^2\right|
= \left|\sum_{\ell=1}^L(a_{\ell}^2-b_{\ell}^2)\right| \\
&\leq& \sum_{\ell=1}^L|a_{\ell}^2-b_{\ell}^2|
\leq 2\sum_{\ell=1}^L|a_{\ell}-b_{\ell}|\max\{|a_{\ell}|,|b_{\ell}|\}\\
&\leq& 2\sum_{\ell=1}^L|a_{\ell}-b_{\ell}|,
\end{eqnarray*}
since $f^{\ell}(\vec x), \frac{1}{k}\sum_i^k\tilde{z}_i^{\ell},\eta^{\ell}(\vec x)\in [0,1].$
%Also, $\int_\calX |a_j-b_j| d\vec x \leq {\rm vol}(\calX) \int_\calX |a_j-b_j|^2
%d\vec x$ by Cauchy-Schwarz.
Therefore,
it suffices to show that
\begin{eqnarray*}
&&\sum_{\ell=1}^L{\mathbb E}_{\mathcal T_m}\left[
\int_{\mathcal X} \left|((f^{\ell}(\vec x)-\frac{1}{k}\sum_i^k\tilde{z}_i^{\ell}) - (f^{\ell}(\vec x)-\eta^{\ell}(\vec x))\right|%^2
d\vec x\right]\\
&&\xrightarrow{m\to\infty} 0,
\end{eqnarray*}
so the result follows if
%$${\mathbb E}_{\mathcal T_m}\left[
%\int_{\mathcal X} |\eta^j(\vec x)-\frac{1}{k}\sum_i^k\tilde{x}_i^{N+j}|^2d\vec x\right]
%\xrightarrow{m\to\infty} 0\ \text{for all $j$'s}.$$
\begin{equation}
\label{eq:stone}
\lim\limits_{m\to\infty}{\mathbb E}_{\mathcal T_m,\vec x}\left[\left|\eta^{\ell}(\vec x)-\frac{1}{k}\sum_i^k\tilde{z}_i^{\ell}%^2
\right|\right]=0\ \text{for all $\ell$}.
\end{equation}
By Jensen's inequality $({\mathbb E}[f])^2\leq {\mathbb E}( f^2)$, ~\eqref{eq:stone} follows if
\begin{equation}\label{eq:stone1}\lim\limits_{m\to\infty}{\mathbb E}_{\mathcal T_m,\vec x}\left[\left(\eta^{\ell}(\vec x)-\frac{1}{k}\sum_i^k\tilde{z}_i^{\ell}\right)^2
\right]=0\ \text{for all $\ell$}.
\end{equation}

%recall that $y_i^{N+1}\in\{0,1\}$ is the $(N+1)$-th component of the $i$-th closest point in $\mathcal T_m$ to $(x,f(x))$,
Let $\eta^{\ell}_{k,m}(\vec x)=\frac{1}{k}\sum_i^k\tilde{z}_i^{\ell}$.  Then $\eta^{\ell}_{k,m}$ is actually an estimate of the class probability $\eta^{\ell}(\vec x)$ by the $k$-Nearest Neighbor rule. Following the proof of Stone's Theorem~\cite{Stone1977,Devroye1996}, if $k\xrightarrow{m\to\infty}\infty$ and $k/m\xrightarrow{m\to\infty}0$, \eqref{eq:stone1} holds for all distributions $P$.
\end{proof}

\clearpage
\section{Notation}
\label{sec:notation}
\begin{eqnarray*}
h_{\vec f}(\vec x) = \argmax\{f^\ell(\vec x),\ell\in\mathcal Y\} &:& \text{plug-in classifier of estimator}\ \vec f:\mathcal X\to\Delta^{L-1}\\
\bbo_{h_{\vec f}(\vec x)\neq y} &=& \left\{ \begin{array}{rr} 1,& h_{\vec f}(\vec x)\neq y,\\
0,& h_{\vec f}(\vec x) = y.\end{array}\right.\\
R_P(\vec f)=\mathbb{E}_P[\bbo_{h_{\vec f}(\vec x)\neq y}]&:&\text{generalization risk for the estimator}\ \vec f\\
\vec\eta(\vec x) = (\eta^1(\vec x),\ldots,\eta^L(\vec x)) &:&
 \text{class probability function:}\ \eta^\ell(\vec x)
 = P(y=\ell|\vec x)\\
R^*_P=R_P(\vec\eta)&:&\text{Bayes risk}\\
(\text{$D$-risk for our }\pdist)\ R_{D,P}(\vec f)&=&\int_{\mathcal X} d^2(\vec f(\vec x), \vec\eta(\vec x))
d\vec x\\
(\text{empirical $D$-risk}) R_{D,\mathcal T_m}(\vec f) = R_{D,\mathcal T_m,k}(\vec f) &=&\int_{\mathcal X} d^2\left(\vec f(\vec x), \frac{1}{k}\sum_{i=1}^k\tilde{\vec z}_i\right)d\vec x\\
&&\text{where $\tilde{\vec z}_i$ is the vector of the last $L$ components of}\\
&&\text{$(\tilde{\vec x}_i,\tilde{\vec z}_i)$, with $\tilde{\vec x}_i$ the $i^{\rm th}$ nearest neighbor of $\vec x$ in $\calT_m$}\\
(\text{volume penalty term})\  \pg(\vec f) &=& \int_{\grf}\dvol\\
R_{D,P,\lambda}(\vec f)=R_{D,P}(\vec f)+\lambda\pg(\vec f)&:&\text{regularized $D$-risk for estimator}\ \vec f\\
R_{D,\mathcal T_m,\lambda}(\vec f)=R_{D,\mathcal T_m}(\vec f)+\lambda\pg(\vec f)&:&\text{regularized empirical $D$-risk for estimator}\ \vec f\\
\vec f_{D,P,\lambda} &=& \text{function attaining the global minimum for}\ R_{D,P,\lambda}\\
R^*_{D,P,\lambda} = R_{D,P,\lambda}(\vec f_{D,P,\lambda}) &:& \text{minimum value for}\ R_{D,P,\lambda}\\
\vec f_{D,\mathcal T_m,\lambda} = \vec f_{D,\mathcal T_m,k,\lambda}&:&\text{function attaining
the global minimum for}\ R_{D,\mathcal T_m,\lambda}(\vec f)
\end{eqnarray*}
Note that we assume $\vec f_{D,P,\lambda}$ and $\vec f_{D,\mathcal T_m,\lambda}$
exist.

\end{document}